\PassOptionsToPackage{numbers,sort&compress}{natbib}
\documentclass[3p,a4paper]{elsarticle}
\usepackage[utf8]{inputenc}
\usepackage{subfigure}
\usepackage{graphicx}

\title{Neural Discovery of Memory and Nonlocal Kernels in Integro-Differential Equations with Constrained Kolmogorov--Arnold Networks}

\author[myUaddress]{Aruzhan Tleubek} 
\author[myUaddress,myUaddress2]{Salah A Faroughi \corref{mycorrespondingauthor}}

\cortext[mycorrespondingauthor]{Corresponding author: salah.faroughi@utah.edu}

\address[myUaddress]{Energy \& Intelligence Lab, Department of Chemical Engineering, University of Utah, Salt Lake City, Utah  84112, USA
}
\address[myUaddress2]{Department of Mechanical Engineering, University of Utah, Salt Lake City, Utah  84112, USA
}

\date{\today}

\usepackage{mwe}
\usepackage{amsmath, amssymb}
\usepackage{algorithm}
\usepackage{float}
\usepackage{mathtools}  
\usepackage{algpseudocode}
\usepackage[table]{xcolor}
\usepackage{algpseudocode}
\usepackage{makecell}

\usepackage{setspace} 
\usepackage{tabularx}
\usepackage{lineno,hyperref}
\usepackage{amsmath}
\usepackage{bm}
\usepackage{amssymb}
\usepackage[]{amsmath}
\usepackage{upgreek}
\usepackage{float}
\usepackage{listings}
\usepackage{xcolor}

\lstdefinestyle{pythonstyle}{
    language=Python,
    basicstyle=\ttfamily\footnotesize,
    keywordstyle=\color{blue}\bfseries,
    commentstyle=\color{gray}\itshape,
    stringstyle=\color{orange},
    numberstyle=\tiny\color{gray},
    numbers=left,
    numbersep=5pt,
    breaklines=true,
    breakatwhitespace=true,
    frame=single,
    frameround=tttt,
    rulecolor=\color{black!30},
    backgroundcolor=\color{black!3},
    showstringspaces=false,
    tabsize=4,
    captionpos=b,
}

\let\today\relax
\makeatletter
\def\ps@pprintTitle{%
    \let\@oddhead\@empty
    \let\@evenhead\@empty
    \def\@oddfoot{\footnotesize\itshape
         {Submitted preprint — July 2026} \hfill\today}%
    \let\@evenfoot\@oddfoot
    }
\makeatother

\usepackage{pgfplots}
\pgfplotsset{compat=1.5}
\usepgfplotslibrary{units}
\usetikzlibrary{spy}
\usetikzlibrary{pgfplots.groupplots}

\usepackage{epstopdf}
\usepackage{units}
\usepackage{lscape}
\usepackage{booktabs}
\usepackage{gensymb}
\usepackage{multirow}

\usepackage{graphicx}
\usepackage{caption}
\usepackage{subcaption}

\usepackage{booktabs}
\usepackage{multirow}
\usepackage{caption}
\usepackage{subcaption}

\usepackage{soul,color}
\soulregister\cite7
\soulregister\ref7
\soulregister\pageref7

\usepackage[flushleft]{threeparttable}
\usepackage{tikz}
\usepackage{placeins}
\usepackage{color}
\usepackage{hyperref}

\usepackage{stackengine}

\usepackage{array}
\usepackage{tabularx}
\usepackage{pdfpages}
\usepackage[numbers,sort&compress]{natbib}
\usepackage{textcomp, gensymb}
\usepackage{booktabs} 

\newtheorem{theorem}{Theorem}[section]

\newtheorem{lemma}[theorem]{Lemma}

\newtheorem{remark}[theorem]{Remark}

\newtheorem{open problem}[theorem]{Open Problem}
\newenvironment{proof}%
{\par\noindent\textbf{Proof.}\ }%
{\hfill$\square$\par}

\begin{document}

\begin{abstract}
Discovering the memory or nonlocal kernel governing an integro-differential equation (IDE) from sparse and noisy observations is an ill-posed inverse problem. Existing identification methods often rely on problem-specific analytical derivations, specialized observation requirements, or restrictive assumptions about the kernel, limiting their applicability across different classes of IDEs. In this work, we propose a differentiable-solver-based framework for discovering memory and nonlocal kernels directly from spatiotemporal observations. Within the solver, the unknown kernel is represented using a constrained Kolmogorov--Arnold Network (KAN) parameterization, with the physical constraints imposed through two different approaches: a Bernstein-polynomial-based Monotone--Convex KAN (MC-KAN), whose coefficient constraints enforce positivity, monotonic decrease, and convexity by construction, and a Chebyshev-based KAN (Cheb-KAN), in which the same properties are encouraged through soft penalty terms. After training, symbolic regression is applied to the learned kernels to obtain interpretable closed-form representations. We evaluate both methods on benchmarks spanning a one-dimensional Volterra equation, a one-dimensional viscoelastic wave partial integro-differential equation, and a two-dimensional nonlocal reaction-diffusion equation with an anisotropic coupled kernel.  For the 1D problems, both methods recover the correct kernel functional form and achieve comparable solution-reconstruction accuracy. In contrast, for the sparse and noisy 2D nonlocal problem, the hard-constrained MC-KAN consistently achieves lower kernel reconstruction errors than the soft-constrained Cheb-KAN.  Our results demonstrate that enforcing physically motivated shape constraints by construction provides greater robustness than soft penalties for multidimensional kernel discovery from sparse and noisy observations.
\end{abstract}


\maketitle
\section{Introduction}\label{sec:Intro}
Many physical systems exhibit memory effects or spatial nonlocality, where the state variables depend on the entire history of the system's evolution and/or on long-range spatial interactions, including viscoelastic polymer relaxation~\cite{dehghan2006solution, dafermos1970asymptotic}, anomalous diffusion in biological membranes~\cite{caputo2008diffusion, metzler2000random, metzler2014anomalous}, and heat conduction with thermal memory~\cite{gurtin1968general, nunziato1971heat}. These responses are modeled by integro-differential equations (IDEs), where a differential operator, as in standard ordinary or partial differential equations (ODEs/PDEs), is coupled with an integral operator. The key element in IDEs is a kernel that weights contributions from past states or spatially distant regions. Often, the kernel is unknown, and identifying it is crucial for solving IDEs efficiently~\cite{schadle2006fast, gao2022kernel}, incorporating it into reduced-order models~\cite{chorin2000optimal, gouasmi2017priori}, and constructing physically interpretable surrogates. However, kernel identification from data remains a challenging problem. 

The inverse problem of identifying memory and nonlocal kernels from observational data has a long history in applied mathematics. Classical approaches typically reformulate the problem as a Volterra integral equation for the unknown kernel and use problem-specific information such as boundary flux measurements~\cite{pandolfi2015identification} or integral overdetermination constraints~\cite{durdiev2022memory}. The resulting equations are then solved using analytical techniques, including eigenfunction expansions~\cite{durdiev2024kernel}, Laplace transforms~\cite{janno2000inverse}, resolvent estimates~\cite{cavaterra1994identifying}, and fixed-point arguments~\cite{durdiev2022memory}, or formulated as regularized optimization problems solved using adjoint-based optimization, Tikhonov regularization~\cite{lamm2000survey}, iterative reconstruction algorithms, deconvolution~\cite{pandolfi2017identification}, and Prony-series fitting~\cite{soussou1970application}. Although rigorous uniqueness and stability results have been established under suitable structural assumptions, many classical reconstruction methods rely on problem-specific analytical derivations, specialized observation requirements, or restrictive assumptions on the kernel to derive inversion formulas and guarantee well-posedness. Moreover, these reconstruction procedures are often tailored to particular classes of IDEs and observation settings, limiting their applicability across diverse physical systems. These limitations motivate flexible data-driven approaches capable of identifying unknown kernel structures directly from observations.

Recent advances in scientific machine learning (SciML) have made it possible to learn governing equations, constitutive relations, and unknown operators directly from observational data across a wide range of ODEs, PDEs, and dynamical systems~\cite{brunton2016sindy, chen2018neural, rudy2017pdefind, lu2021deeponet, faroughi2026symbolic, koenig2024kan, mostajeran2025minpo}. The same data-driven approach has increasingly been applied to nonlocal systems to learn solution dynamics, kernel functions, or both. Existing work can be broadly grouped into neural-network-based methods, sparse and symbolic regression methods, and neural differential equations. Neural-network-based approaches, often within physics-informed learning frameworks, parameterize the unknown kernel as a neural network and identify its parameters by enforcing the governing equation during optimization. For example, \citet{difonzo2025physics} used an RBF-based physics-informed neural network to recover peridynamic kernels in a nonlocal wave equation, while \citet{khaldi2026learning} studied memory-kernel identification in viscoelasticity using a physics-informed neural architecture with a differentiable Prony state-space memory layer. These methods demonstrate that neural representations can reconstruct kernels in nonlocal and history-dependent systems, but the learned kernels are represented implicitly in network parameters or state-space parameterizations rather than recovered as closed-form expressions. To obtain more interpretable representations, sparse and symbolic regression have also been developed. Building on sparse identification of nonlinear dynamics (SINDy), \citet{breda2025sparse} recovered kernels in delay IDEs and renewal equations, and \citet{carrillo2025sparse} estimated nonlocal interaction kernels using sparse identification with partial inversion. Symbolic-regression approaches, such as the method of \citet{yu6404003physics}, combine preprocessing or transform-based steps with symbolic search to identify candidate kernel forms. These approaches can produce compact, interpretable expressions when the target kernel is well represented by the chosen candidate space, but their performance may depend strongly on preprocessing, prescribed libraries, or transform-based assumptions. A third related class of methods is neural ODEs and universal differential equations, in which unknown components of a dynamical system are replaced by trainable neural networks. Extending this idea to nonlocal dynamics, \citet{zappala2023neural} introduced neural integro-differential equations (NIDE) and showed that incorporating integral terms can improve the modeling of nonlocal brain dynamics compared with standard Neural ODEs. Their formulation follows an optimize-then-discretize~\cite{chen2018neural} strategy that requires deriving continuous adjoint equations specific to the IDE for gradient computation. A subsequent extension introduced Neural Integral Equations (NIEs) and Attentional Neural Integral Equations (ANIEs), in which integral operators are approximated via Monte Carlo sampling and attention mechanisms to improve scalability in higher-dimensional settings, with applications to Navier--Stokes equations and brain dynamics~\cite{zappala2024learning}. However, these methods learn the solution dynamics of the IDE, rather than recover the memory kernel itself as an explicit object of scientific interest. They also remain tied to case-specific formulations, requiring either IDE-specific adjoint derivations or iterative neural integral solvers. 

In this work, we instead target the identification of unknown memory and nonlocal kernels in IDEs directly from observations. Our framework replaces the unknown kernel with a trainable neural representation embedded within a differentiable IDE solver, and optimizes its parameters by minimizing the mismatch between predicted and observed spatiotemporal data. In contrast to adjoint-based methods, we use a discretize-then-optimize strategy~\cite{gholami2019anode, onken2020discretize}  where the gradients are computed directly through the numerical solver using automatic differentiation, removing the need for case-specific adjoint derivations. The framework is therefore compatible with standard time integrators and spatial discretizations. Kernel identification from sparse and noisy observations is generally ill-posed~\cite{engl1996regularization}, so physically meaningful constraints are required to restrict the solution space. In a broad class of memory-dominated systems, including the viscoelastic relaxation, anomalous diffusion, and thermal memory examples noted above, the kernel encodes a fading memory in which the influence of past states decays smoothly with elapsed time. We therefore focus on kernels that are positive, monotonically decreasing, and convex, consistent with standard exponential and power-law relaxation models~\cite{soussou1970application, metzler2000random, hanyga2018simple, luchko2020complete, hanyga2013wave}. To impose these properties, we consider two constrained Kolmogorov--Arnold Network (KAN) parameterizations. The proposed Monotone--Convex KAN (MC-KAN) is a Bernstein-polynomial-based architecture that enforces positivity, monotonic decrease, and convexity by construction through simple constraints on its polynomial coefficients. The second is a Chebyshev-based KAN (Cheb-KAN)~\cite{ss2024chebyshev}, in which the same physical properties are encouraged through soft penalty terms during optimization. After training, the learned neural kernel is converted into an interpretable closed-form expression using symbolic regression~\cite{cranmer2023pysr}. We validate the proposed framework on three benchmark problems spanning temporal and spatial kernel discovery: a Volterra IDE, a one-dimensional viscoelastic wave partial integro-differential equation (PIDE), and a two-dimensional nonlocal reaction-diffusion equation. Across these benchmarks, both KAN parameterizations successfully identify the underlying kernels, while the hard-constrained formulation becomes advantageous for sparse, noisy, and multidimensional inverse problems.

The remainder of this paper is organized as follows. Section~\ref{sec:methods} presents the methodology, including the problem formulation, memory kernel properties, and the MC-KAN framework. Section~\ref{sec:results} presents the numerical results and discussion. Section~\ref{sec:conclusion} concludes with key findings and future directions. \ref{sec-appendix:theorem} presents the theoretical guarantee showing that MC-KAN preserves the prescribed kernel constraints,~\ref{sec-appendix:implementation} describes the coefficient reparameterization, and \ref{sec:appendix-numerical} provides additional details on the numerical experiments.

\section{Methodology}\label{sec:methods}
In this section, we present the proposed framework for kernel discovery, illustrated in Fig.~\ref{fig:fig1}. We first formulate the kernel discovery problem and define the observational data. We then introduce the two KAN parameterizations considered in this work, namely the hard-constrained MC-KAN and the soft-constrained Cheb-KAN. Finally, we describe the differentiable forward solver, the training procedure, and the symbolic regression step used to obtain an interpretable closed-form representation of the learned kernel.

\subsection{Problem Setup}
Let $\mathcal{D}=\{(t_k,\mathbf{x}_i,u(t_k,\mathbf{x}_i))\}$, $k=1,\ldots,N_t$, $i=1,\ldots,N_x$, denote a dataset of spatiotemporal observations of a scalar field $u:[0,T]\times\Omega\rightarrow\mathbb{R}$, where $\Omega\subset\mathbb{R}^n$ is the spatial domain; see the problem setup panel in Fig.~\ref{fig:fig1}. We assume that $u$ satisfies an IDE of the form, 
\begin{equation}
    \mathcal{T}[u](t,\mathbf{x})
    =
    \mathcal{N}[u](t,\mathbf{x})
    +
    \int_0^t \int_{\Omega}
    K(t-s,\mathbf{x}-\mathbf{y})\,
    \mathcal{F}[u](s,\mathbf{y})
    \, d\mathbf{y}\, ds
    +
    S(t,\mathbf{x}),
    \label{eq:general-IDE-method}
\end{equation}
where $\mathcal{T}[u](t,\mathbf{x})$ is the temporal evolution operator, $\mathcal{N}[u](t,\mathbf{x})$ is the local operator, $S(t,\mathbf{x})$ is an external source term, and $K(t-s,\mathbf{x}-\mathbf{y})$ is a spatiotemporal convolution kernel that weights the contribution of the field state at past time $s$ or spatial location $\mathbf{y}$. The operator $\mathcal F$ is a (possibly differential) operator acting on $u$ and evaluated at $(s,\mathbf y)$. Equation~\eqref{eq:general-IDE-method} is supplemented with boundary and initial conditions,
\begin{equation}
    \mathcal{B}[u](t,\mathbf{x}) = 0,
    \qquad \mathbf{x} \in \partial\Omega,
    \quad t \in [0,T],
    \label{eq:bc-method}
\end{equation}
\begin{equation}
    u(0,\mathbf{x}) = u_0(\mathbf{x}),
    \qquad
    \partial_t u(0,\mathbf{x})
    = v_0(\mathbf{x}),
    \qquad \mathbf{x} \in \Omega,
    \label{eq:ic-method}
\end{equation}
where $\mathcal{B}$ is a boundary operator, $u_0(\mathbf{x})$ and $v_0(\mathbf{x})$ are the prescribed initial
field values and initial rates of change, respectively. The second condition is imposed only when the temporal operator $\mathcal{T}$ is second order in time. The functional form of the kernel $K(t-s,\mathbf{x}-\mathbf{y})$ encodes how strongly the current state is influenced by past states and/or spatial locations, and is assumed to be \textit{unknown}. We seek to identify $K(t-s,\mathbf{x}-\mathbf{y})$ directly from the spatiotemporal dataset $\mathcal{D}$. 

\begin{figure}[h!]
    \centering
    \includegraphics[width=\linewidth]{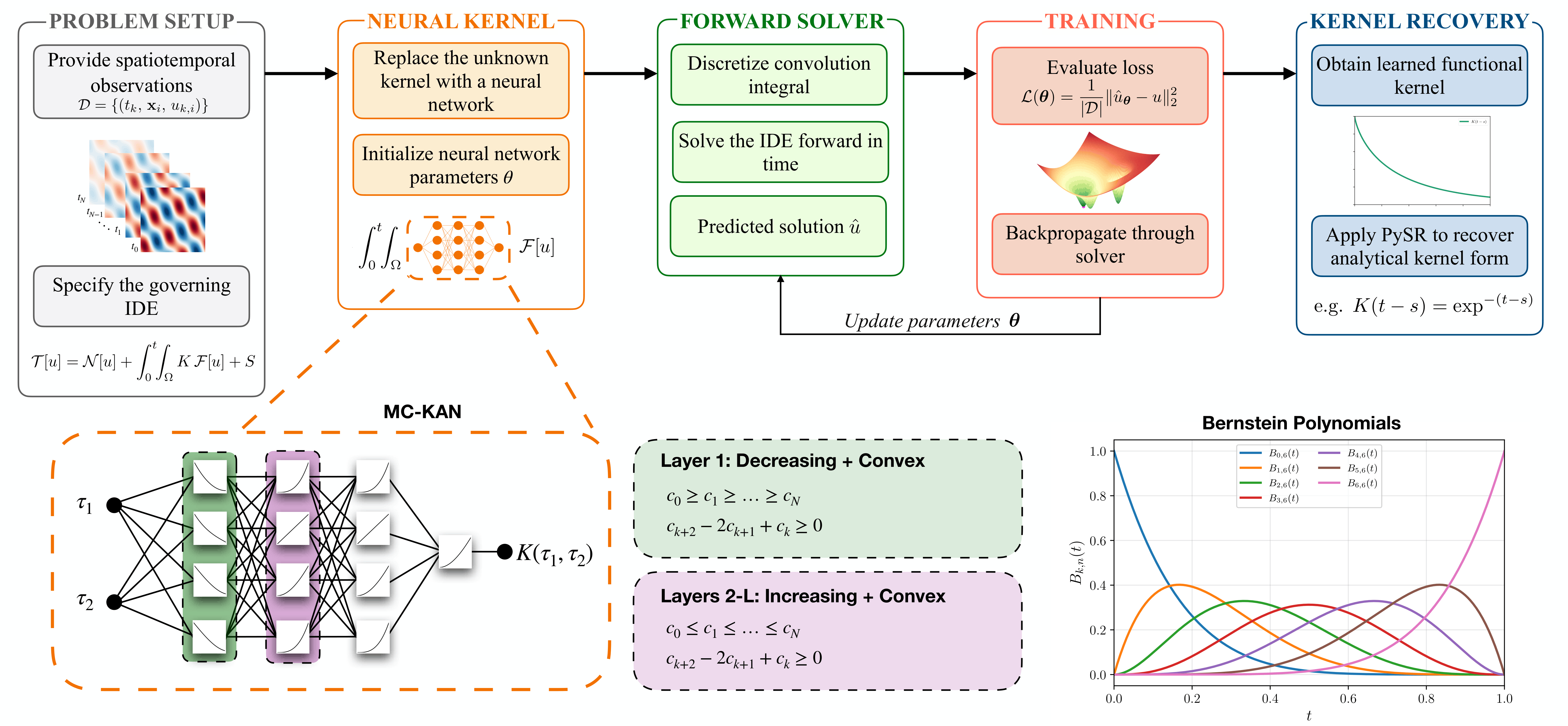}
    \caption{Schematic of the neural-network-embedded kernel discovery framework. \textbf{Top:} The framework consists of five stages. Given the spatiotemporal dataset $\mathcal{D}$ and the governing IDE (\emph{Problem Setup}), the unknown kernel is replaced by a neural network embedded inside the memory integral (\emph{Neural Kernel}); the IDE is discretized and solved forward in time to produce the predicted field $\hat{u}$ (\emph{Forward Solver}); the parameters of the network $\boldsymbol{\theta}$ are updated by minimizing the data mismatch $\mathcal{L}(\boldsymbol{\theta})$ with gradients obtained by backpropagation through the solver (\emph{Training}); and the converged kernel is distilled into a closed-form expression via symbolic regression (PySR) (\emph{Kernel Recovery}). \textbf{Bottom:} the MC-KAN parameterization of the kernel, a KAN whose edge functions are Bernstein polynomials (right). The physical constraints are enforced through conditions on the Bernstein coefficients. The first layer uses a decreasing, convex ordering and the remaining layers an increasing, convex ordering, so that the composed kernel $K(\boldsymbol{\tau})$ is positive, monotonically decreasing, and convex along each lag dimension.}
    \label{fig:fig1}
\end{figure}

When $\mathcal{D}$ is sparse or noisy, this identification problem is ill-posed in the sense of Hadamard~\cite{engl1996regularization}, and prior information is required to obtain a stable, physically meaningful solution. We therefore restrict attention to a broad class of dissipative memory processes whose kernels are positive, monotonically decreasing, and convex along each lag dimension, rather than jointly convex over the full lag domain. Writing $\boldsymbol{\tau} = (\tau_t,\boldsymbol{\tau}_X)$, with $\tau_t=t-s, \boldsymbol{\tau}_X=|\mathbf{x}-\mathbf{y}|$
(componentwise absolute value), and $\tau_j$ for its
$j$-th component, we require, 
\begin{equation}
    K(\boldsymbol{\tau}) > 0,
    \qquad
    \frac{\partial K}{\partial \tau_j} \le 0,
    \qquad
    \frac{\partial^2 K}{\partial \tau_j^2} \ge 0,
    \qquad j = 1,\ldots,n+1 ,
    \label{eq:kernel-properties}
\end{equation}
on the lag domain, taken with $\tau_t^{\min} > 0$ so that the weakly singular power-law and Mittag--Leffler kernels remain permissible while the derivatives in Eq.~\eqref{eq:kernel-properties} stay well defined.

\subsection{Neural Kernel Parameterization}
To identify a kernel satisfying Eq.~\eqref{eq:kernel-properties}, we replace the unknown kernel with a neural network $K_{\boldsymbol{\theta}}$ (see neural kernel in Fig.~\ref{fig:fig1}), where $\boldsymbol{\theta}$ denotes the trainable parameters,
\begin{equation}
    K_{\boldsymbol{\theta}}
    :
    \mathbb{R}_{\ge 0}\times\mathbb{R}^n
    \;\to\;
    \mathbb{R},
    \qquad
    \boldsymbol{\tau} \mapsto K_{\boldsymbol{\theta}}(\boldsymbol{\tau}),
    \label{eq:kernel-nn}
\end{equation}
that maps a space--time lag $\boldsymbol{\tau}$ to a scalar kernel value. In this work, we consider two KAN-based parameterizations. The first is a Chebyshev-based KAN (Cheb-KAN)~\cite{ss2024chebyshev, shukla2024comprehensive}, in which positivity, monotonic decrease, and convexity are encouraged through soft penalty terms. The second is the proposed Monotone--Convex Kolmogorov--Arnold Network (MC-KAN), in which the same physical properties are enforced by construction. Our choice of KANs is motivated by their decomposition into learnable univariate edge functions \cite{faroughi2026kolmogorov,faroughi2026neural,mostajeran2025scaled, yu2024kan}, which makes shape constraints easier to impose. Standard multilayer perceptrons (MLPs)~\cite{faroughi2024physics}, on the other hand, often require separate specialized architectures such as input-convex~\cite{amos2017input} or monotone neural networks~\cite{sill1997monotonic}.  Recent work has shown that shape constraints can be imposed in KANs through conditions on spline coefficients, including input-convex KANs for polyconvex hyperelastic constitutive modeling~\cite{thakolkaran2025ickan} and convex approximation and optimal transport~\cite{deschatre2026ickan}. Building on this idea, MC-KAN adapts constrained KANs to the memory-kernel setting by using a Bernstein basis~\cite{farouki2012bernstein} whose coefficient constraints guarantee positivity, monotone decrease, and convexity along each lag dimension.

\subsubsection{Soft-constrained Chebyshev KAN}
For a Cheb-KAN architecture $[d,h_1,\ldots,h_{L-1},h_L]$, where $d=n+1$ is the number of kernel inputs (one temporal lag and $n$ spatial lags) and $h_l$ is the number of nodes in layer $l$, with $h_0=d$ and $h_L=1$, the neural kernel is constructed as a composition of univariate transformations~\cite{liu2025kan, shukla2024comprehensive},
\begin{equation} 
    K_{\text{Cheb-KAN}}(\boldsymbol{\tau}) = \left( \boldsymbol{\Phi}_L \circ \tanh \circ \boldsymbol{\Phi}_{L-1} \circ \cdots \circ \tanh \circ \boldsymbol{\Phi}_1 \right)\!\bigl(\boldsymbol{\chi}(\boldsymbol{\tau})\bigr), 
    \label{eq:cheb-kernel} 
\end{equation}
where $\boldsymbol{\chi}(\boldsymbol{\tau})$ denotes the normalized lag input and $\tanh$ acts componentwise. Since Chebyshev polynomials are defined on $[-1,1]$, each lag component $\tau_i$ is linearly normalized as
\begin{equation} 
    \chi_i(\boldsymbol{\tau}) = 2\,\frac{\tau_i-\tau_i^{\min}}{\tau_i^{\max}-\tau_i^{\min}} - 1, 
    \qquad i=1,\ldots,d, 
\label{eq:cheb-normalization} 
\end{equation}
so that $\chi_i\in[-1,1]$. To ensure numerical stability during training and to respect the domain of the Chebyshev polynomials, a $\tanh$ activation is applied after each hidden layer, constraining all intermediate values to $[-1,1]$. Each layer map $\boldsymbol{\Phi}_l:\mathbb{R}^{h_{l-1}}\rightarrow\mathbb{R}^{h_l}$ is composed of learnable univariate edge functions. For $\mathbf{z}\in\mathbb{R}^{h_{l-1}}$,
\begin{equation}
    \left(\boldsymbol{\Phi}_l(\mathbf{z})\right)_p = \sum_{q=1}^{h_{l-1}} \phi^{(l)}_{p,q}(z_q), 
    \qquad p=1,\ldots,h_l, 
    \label{eq:cheb-layer} 
\end{equation}
and each edge function is parameterized in a Chebyshev polynomial basis~\cite{ss2024chebyshev},
\begin{equation} 
    \phi^{(l)}_{p,q}(z) = \sum_{k=0}^{M} a^{(l)}_{p,q,k}\,T_k(z), 
    \qquad z\in[-1,1], 
\label{eq:cheb-edge} 
\end{equation}
where $z$ denotes the scalar input to the edge, $M$ is the maximum polynomial degree, $T_k$ is the Chebyshev polynomial of the first kind of degree $k$, and $a^{(l)}_{p,q,k}\in\mathbb{R}$ are trainable coefficients, so that $\boldsymbol{\theta}_{\text{Cheb-KAN}}=\{a^{(l)}_{p,q,k}\}$. The Chebyshev polynomials are defined recursively as \cite{mostajeran2024epi, guo2025physics},
\begin{equation} 
    T_0(z)=1, \qquad T_1(z)=z, \qquad T_{k+1}(z)=2z\,T_k(z)-T_{k-1}(z), \qquad k\ge1. 
    \label{eq:cheb-recursion} 
\end{equation}

\subsubsection{Hard-constrained Monotone--Convex KAN}\label{sec:mckan}
For the MC-KAN architecture $[d,h_1,\ldots,h_{L-1},h_{L}]$, where $d=n+1$ is the number of kernel inputs (one temporal lag and $n$ spatial lags) and $h_l$ is the number of nodes in layer $l$, the neural kernel is written as~\cite{liu2025kan},
\begin{equation} 
    K_{\text{MC-KAN}}(\boldsymbol{\tau}) = \left( \boldsymbol{\Phi}_L \circ \boldsymbol{\Phi}_{L-1} \circ \cdots \circ \boldsymbol{\Phi}_1 \right)\!\bigl(\boldsymbol{\zeta}(\boldsymbol{\tau})\bigr), 
    \label{eq:mc-kernel} 
\end{equation}
where $\boldsymbol{\zeta}(\boldsymbol{\tau})$ denotes the normalized lag input. Since Bernstein polynomials are defined on $[0,1]$, each lag component $\tau_i$ is first normalized as,
\begin{equation} 
    \zeta_i(\boldsymbol{\tau}) = \left( \frac{\tau_i-\tau_i^{\min}}{\tau_i^{\max}-\tau_i^{\min}} \right)^{\alpha_i}, 
    \qquad i=1,\ldots,d. 
    \label{eq:bernstein-normalization} 
\end{equation}
Here, $\alpha_i\in(0,1]$ is a scaling exponent that adjusts the resolution of the normalized coordinate along the $i$-th lag dimension. Depending on the application, $\alpha_i$ may be specified a priori or treated as a trainable parameter learned jointly with the network weights. Linear normalization corresponds to $\alpha_i=1$, while $\alpha_i<1$ gives a power-law normalization that allocates more resolution near small lags. Restricting $\alpha_i\le1$ makes each normalized coordinate $\zeta_i$ monotone increasing and concave in $\tau_i$, so the normalization preserves positivity, monotone decrease, and convexity under composition. When $\alpha_i$ is learnable it is constrained to be $\alpha_i\in(0,1)$. A formal argument is provided in \ref{sec-appendix:theorem}.

Let $\mathbf{z}^{(l)}\in\mathbb{R}^{h_l}$ denote the output of layer $l$, with $\mathbf{z}^{(0)}=\boldsymbol{\zeta}(\boldsymbol{\tau})\in[0,1]^d$ and $h_0=d$, $h_L=1$. Each layer map $\boldsymbol{\Phi}_l:[0,1]^{h_{l-1}}\rightarrow(0,1)^{h_l}$ is composed of learnable univariate edge functions, aggregated as an average~\cite{liu2025kan},
\begin{equation} 
    z_p^{(l)} = \left(\boldsymbol{\Phi}_l\!\left(\mathbf{z}^{(l-1)}\right)\right)_p = \frac{1}{h_{l-1}} \sum_{q=1}^{h_{l-1}} \phi_{p,q}^{(l)}\!\left(z_q^{(l-1)}\right), \qquad p=1,\ldots,h_l, 
    \label{eq:mc-kan-layer} 
\end{equation}
for $l=1,\ldots,L$. Here, $q$ indexes the input node from layer $l-1$, and $\phi_{p,q}^{(l)}:[0,1]\rightarrow(0,1)$ is a learnable univariate function on the edge from node $q$ in layer $l-1$ to node $p$ in layer $l$. The coefficient parameterization introduced below constrains each edge function to map into $(0,1)$. Since each layer computes the average of its incoming edge-function values, every layer output, including the final kernel value, remains in $(0,1)$. The factor $1/h_{l-1}$ is therefore a deliberate normalization for bound control, in contrast to the unnormalized sum used in standard KANs. This boundedness is an architectural consequence rather than a physical requirement. If kernels with arbitrary amplitude are required, the final network output may be multiplied by a positive learnable scaling parameter without affecting the positivity, monotonicity, or convexity guarantees.

The key design choice in MC-KAN is to parameterize each edge function as a Bernstein polynomial of degree $N$,
\begin{equation} 
    \phi_{p,q}^{(l)}(z) = \sum_{k=0}^{N} c_k^{(l,p,q)}\,B_{k,N}(z), 
    \qquad z\in[0,1], 
    \label{eq:edge-bernstein} 
\end{equation}
where $z=z_q^{(l-1)}$ is the scalar input to the edge, and the Bernstein basis functions are
\begin{equation} 
    B_{k,N}(z) = \binom{N}{k} z^{k}(1-z)^{N-k}, \qquad k=0,\ldots,N, \quad z\in[0,1]. 
    \label{eq:bernstein-basis} 
\end{equation}
Here, $\binom{N}{k}=N!/[k!(N-k)!]$ is the binomial coefficient. Because the Bernstein basis satisfies $B_{k,N}(z)\ge0$ for all $z\in[0,1]$~\cite{farouki2012bernstein}, shape constraints can be imposed through conditions on the Bernstein coefficients. For a Bernstein polynomial,
\[ \phi(z) = \sum_{k=0}^{N} c_k\,B_{k,N}(z), \]
the first and second derivatives are,
\begin{equation} 
    \frac{d\phi}{dz} = N \sum_{k=0}^{N-1} (c_{k+1}-c_k)\,B_{k,N-1}(z), 
    \label{eq:derivative} 
\end{equation}
\begin{equation} 
    \frac{d^2\phi}{dz^2} = N(N-1) \sum_{k=0}^{N-2} (c_{k+2}-2c_{k+1}+c_k)\,B_{k,N-2}(z). 
    \label{eq:second-derivative} 
\end{equation}

Since $B_{k,N-1}(z)\ge0$ and $B_{k,N-2}(z)\ge0$ on $[0,1]$, sufficient conditions for monotonicity and convexity are obtained by constraining the first and second finite differences of the coefficients. Thus, a Bernstein edge function satisfies the desired boundedness, positivity, monotone decrease, and convexity properties if its coefficients satisfy the conditions listed in Table~\ref{tab:constraints}. The coefficient conditions in Table~\ref{tab:constraints} are imposed using a differentiable constrained reparameterization of the Bernstein coefficients. In the first layer, the coefficients are ordered as $c_0\ge c_1\ge\cdots\ge c_N$ to obtain decreasing convex edge functions. In subsequent layers, the coefficient sequence is reversed, giving $c_0\le c_1\le\cdots\le c_N$ and hence increasing convex edge functions. In all layers, the reparameterization keeps the coefficients in $(0,1)$, so each edge function maps $[0,1]$ into $(0,1)$. This construction allows the prescribed kernel constraints to propagate through the full network depth. The formal proof and coefficient reparameterization are provided in \ref{sec-appendix:theorem} and \ref{sec-appendix:implementation}.

\begin{table}[h]
\centering
\caption{Sufficient Bernstein coefficient conditions for a Bernstein edge function to satisfy the boundedness and shape properties used in MC-KAN.}
\label{tab:constraints}
\begin{tabular}{llc}
\toprule
Property & Edge condition & Sufficient coefficient condition \\
\midrule
Boundedness and positivity &
$0<\phi(z)<1$ &
$0<c_k<1 \quad \forall\,k$ \\[4pt]
Monotone decrease &
$\dfrac{d\phi}{dz}\le0$ &
$c_0\ge c_1\ge\cdots\ge c_N$ \\[4pt]
Convexity &
$\dfrac{d^2\phi}{dz^2}\ge0$ &
$c_{k+2}-2c_{k+1}+c_k\ge0
\quad
\forall\,k=0,\ldots,N-2$ \\
\bottomrule
\end{tabular}
\end{table}

\subsection{Training Procedures}
Given a neural kernel $K_{\boldsymbol{\theta}}$, we embed it in a differentiable numerical solver for Eq.~\eqref{eq:general-IDE-method}; see the forward solver in Fig.~\ref{fig:fig1}. The integral term is discretized using a quadrature or convolution scheme appropriate to the governing equation. Depending on the structure of the kernel and the problem geometry, the resulting discrete integral operator is evaluated either directly or using Fast Fourier Transforms (FFTs). The temporal evolution operator is then advanced with a suitable time-integration scheme. The specific spatial discretizations, quadrature rules, convolution evaluations, and time integrators are chosen according to the structure of each benchmark problem and reported where necessary. During the forward solve, the neural kernel $K_{\boldsymbol{\theta}}$ is evaluated at the required lag points to assemble the discrete integral operator and propagate the solution in time. Let $\hat{u}(t_k,\mathbf{x}_i;K_{\boldsymbol{\theta}})$ denote the numerical solution obtained with the current neural kernel, evaluated at the observation points in $\mathcal{D}$. The data loss is defined as,
\begin{equation}
    \mathcal{L}_{\mathrm{data}}(\boldsymbol{\theta})
    =
    \frac{1}{|\mathcal{D}|}
    \sum_{(t_k,\mathbf{x}_i)\in\mathcal{D}}
    \left(
        \hat{u}(t_k,\mathbf{x}_i;K_{\boldsymbol{\theta}})
        -
        u(t_k,\mathbf{x}_i)
    \right)^2 .
    \label{eq:loss-method}
\end{equation}

For MC-KAN, the kernel constraints are enforced by construction through the Bernstein coefficient parameterization, so the training objective is simply $\mathcal{L}_{\mathrm{data}}$. For the Cheb-KAN baseline, the same data loss is combined with soft constraint penalties evaluated at a finite set of sampled lag points. Let $\{\boldsymbol{\tau}_j\}_{j=1}^{N_g}$ denote points in the lag domain used to evaluate constraint violations. Depending on the experiment, these points are selected either from a uniform grid or sampled randomly from the domain. We define,
\begin{equation}
    \mathcal{L}_{\mathrm{pos}}
    =
    \frac{1}{N_g}
    \sum_{j=1}^{N_g}
    \left[
        \max
        \bigl(
            0,
            -K_{\boldsymbol{\theta}}(\boldsymbol{\tau}_j)
        \bigr)
    \right]^2,
    \label{eq:cheb-pos-loss}
\end{equation}
\begin{equation}
    \mathcal{L}_{\mathrm{mono}}
    =
    \frac{1}{N_g d}
    \sum_{j=1}^{N_g}
    \sum_{i=1}^{d}
    \left[
        \max
        \left(
            0,
            \frac{\partial K_{\boldsymbol{\theta}}}
                 {\partial \tau_i}
            (\boldsymbol{\tau}_j)
        \right)
    \right]^2,
    \label{eq:cheb-mono-loss}
\end{equation}
\begin{equation}
    \mathcal{L}_{\mathrm{conv}}
    =
    \frac{1}{N_g d}
    \sum_{j=1}^{N_g}
    \sum_{i=1}^{d}
    \left[
        \max
        \left(
            0,
            -
            \frac{\partial^2 K_{\boldsymbol{\theta}}}
                 {\partial \tau_i^2}
            (\boldsymbol{\tau}_j)
        \right)
    \right]^2.
    \label{eq:cheb-conv-loss}
\end{equation}
Here, $\mathcal{L}_{\mathrm{pos}}$ penalizes negative kernel values, $\mathcal{L}_{\mathrm{mono}}$ penalizes positive first derivatives, and $\mathcal{L}_{\mathrm{conv}}$ penalizes negative second derivatives along each lag dimension. The full Cheb-KAN training loss is then,
\begin{equation}
    \mathcal{L}_{\mathrm{Cheb-KAN}}
    =
    \mathcal{L}_{\mathrm{data}}
    +
    \lambda
    \left(
    w_{\mathrm{pos}}
    \mathcal{L}_{\mathrm{pos}}
    +
    w_{\mathrm{mono}}
    \mathcal{L}_{\mathrm{mono}}
    +
    w_{\mathrm{conv}}
    \mathcal{L}_{\mathrm{conv}}
    \right),
    \label{eq:cheb-total-loss}
\end{equation}
where $\lambda$ controls the overall strength of the constraint regularization, and $w_{\mathrm{pos}}$, $w_{\mathrm{mono}}$, and $w_{\mathrm{conv}}$ determine the relative weights of the positivity, monotonicity, and convexity penalties, respectively. In all experiments, these three weights are set equal. The value of $\lambda$ is selected separately for each benchmark by a short calibration procedure: it is chosen so that, at the start of training, the total constraint penalty is of the same order of magnitude as the data loss. During training, the constraint loss is monitored to verify that it remains active and does not become negligible relative to the data loss. Gradients of the loss with respect to the trainable parameters $\boldsymbol{\theta}$ are computed by automatic differentiation~\cite{baydin2018automatic, paszke2019pytorch} through the discrete solver. Thus, the method follows a discretize-then-optimize strategy. Unless otherwise specified, we train the parameters using Adam~\cite{kingma2014adam} followed by L-BFGS with a strong Wolfe line search~\cite{liu1989limited,more1994line}.

Once training is complete, the neural kernel $K_{\boldsymbol{\theta}}$ provides a pointwise approximation of the underlying kernel over its domain. To extract a closed-form expression, we sample the learned kernel and apply symbolic regression via PySR~\cite{cranmer2023pysr} (see kernel recovery in Fig.~\ref{fig:fig1}), which searches over algebraic expressions to fit the sampled values. Alternative sparse-regression approaches over a predefined function
library, such as SINDy~\cite{brunton2016sindy}, could also be applied at this stage. The recovered expression yields a physically interpretable kernel that can be
compared directly with known analytical forms when available.

\section{Results and Discussion}\label{sec:results}

In this section, we evaluate the proposed framework on three kernel identification problems involving representative memory and nonlocal kernels commonly encountered in applications of IDEs. For each problem, the objective is to recover the unknown kernel from observations of the system response using only the governing equation and the available data. To compare hard and soft constraint enforcement, both MC-KAN and Cheb-KAN use the same number of hidden layers and approximately the same number of trainable parameters. All experiments are repeated over five independent random seeds to account for sensitivity to initialization. Unless otherwise stated, the reported results correspond to the mean performance across all runs, with standard deviations reported where relevant. To quantify predictive accuracy, we use the relative $L_2$ error,
\begin{equation}
    \mathcal{E}(\nu)
    =
    \frac{\|\hat{\nu}-\nu\|_2}
         {\|\nu\|_2},
    \label{eq:rel_error}
\end{equation}
where $\nu$ denotes the true quantity and $\hat{\nu}$ denotes its learned approximation. This metric is applied to both the recovered kernel $K$ and the predicted solution $u$ throughout all experiments. When the training observations are noisy, the ground truth for $\mathcal{E}(u)$ is the reference solution before the synthetic noise is added, evaluated at the observation points in $\mathcal{D}$.

\subsection{Experiment I: 1D Volterra IDE}
As a first experiment, we consider a scalar Volterra IDE describing the temporal evolution of a state variable $u(t)$ in a system with memory. Such equations arise in models of viscoelasticity, population dynamics with hereditary effects, and anomalous relaxation processes~\cite{brunner2017volterra, cushing2013integrodifferential}. The governing equation is,

\begin{equation}
    \frac{du(t)}{dt}
    +
    u(t)
    =
    \kappa \int_0^t
    K(t-s)\,
    u(s)\,ds,
    \qquad
    t\in(0,T],
    \label{eq:volterra}
\end{equation}
subject to the initial condition $u(0)=1$, and $\kappa>0$ controls the memory strength. Throughout this experiment, we set $\kappa=1.5$ and $T=5\,\mathrm{s}$. To generate reference data, we prescribe the memory kernel,
\begin{equation}
    K(t-s)
    =
    e^{-(t-s)}.
    \label{eq:exp_kernel}
\end{equation}

The reference data are generated by solving Eq.~\eqref{eq:volterra} using a forward Euler time-discretization scheme together with a trapezoidal approximation of the memory integral. A sufficiently small time step, $\Delta t=5\times10^{-3} \mathrm{s}$, is used for time integration, for which the numerical solution was observed to be stable and converged; see Fig.~\ref{fig:appendix-p1-convergence} in Appendix. The resulting solution is sampled at $N_t=1001$ uniformly spaced time points. 

In this experiment, we investigate the effect of measurement noise on kernel discovery. We take $501$ uniformly spaced samples from the reference solution and add independent Gaussian noise,
\begin{equation}
u^{\mathrm{obs}}(t_i)
=
u(t_i)
+
\sigma \,\varepsilon_i,
\qquad
\varepsilon_i \sim \mathcal{N}(0,1),
\end{equation}
where $\sigma$ is the noise standard deviation. The resulting training
dataset is
$
\mathcal{D}
=
\bigl\{\,\bigl(t_i,\,u^{\mathrm{obs}}(t_i)\bigr)\,\bigr\}_{i=1}^{501}.
$

The unknown kernel is represented by either MC-KAN or Cheb-KAN. In Cheb-KAN, positivity, monotonic decrease, and convexity are encouraged through equally weighted soft penalty terms with $\lambda=10^{-1}$. For both models, the inputs are linearly normalized. Following the original Cheb-KAN formulation~\cite{ss2024chebyshev}, we use cubic Chebyshev edge functions ($M=3$). For MC-KAN, the Bernstein degree is fixed at $N=6$, which was selected empirically to provide sufficient approximation accuracy while keeping the number of trainable parameters moderate. During training, the differentiable solver uses the same time step and numerical discretization as the reference solver, advancing on the full temporal grid ($N_t = 1001$). At each optimization step, we compute the solution on this grid and evaluate the loss only at the $501$ observation times in $\mathcal{D}$. For evaluation, the learned kernel is sampled on the same 501-point temporal grid used for the training observations, and the kernel reconstruction error is computed on this common grid. Complete numerical, architectural, and optimization settings are summarized in Table~\ref{tab:exp1-config}. For each method, we evaluate three architectures of increasing depth (shallow, medium, and deep), with widths chosen to match parameter counts across the two methods, so that the comparison is not tied to a single network choice.

Table~\ref{tab:p1_tab1} reports the solution and kernel reconstruction errors for all six network configurations at each noise level. Both MC-KAN and Cheb-KAN successfully identify the unknown kernel across all architectures and noise levels. Even at $\sigma=0.15$, both methods recover the kernel with relative errors below $13\%$. The solution reconstruction errors are nearly identical for the two methods, increasing from almost zero in the noise-free case to below $2\%$ at the highest noise level. Among the six network configurations, however, the lowest kernel reconstruction error at every noise level is achieved by MC-KAN. For clarity, Fig.~\ref{fig:p1-fig1} compares the two methods using the common architecture depth that produces the lowest mean kernel reconstruction error among all six configurations at each noise level. In other words, the shallow, medium, or deep architecture associated with the lowest kernel error is identified from Table~\ref{tab:p1_tab1}, and both MC-KAN and Cheb-KAN are then evaluated using that same depth. Relative to Cheb-KAN, MC-KAN reduces $\mathcal{E}(K)$ by approximately $54\%$ in the noise-free case ($0.85\%$ versus $1.83\%$), $32\%$ at $\sigma=0.05$ ($4.47\%$ versus $6.62\%$), $32\%$ at $\sigma=0.10$ ($6.83\%$ versus $10.05\%$), and $13\%$ at $\sigma=0.15$ ($8.78\%$ versus $10.09\%$). Figures~\ref{fig:p1-fig1}(b) and (c) show that both methods accurately reconstruct the solution at low and high noise levels. 

\begin{table}[h!]
\centering
\caption{Results for Experiment I, showing the solution and kernel reconstruction errors under additive Gaussian noise across three network depths ($L=1,2,3$ hidden layers). MC-KAN uses polynomial degree $N=6$ with architectures $[1,4,1]$, $[1,4,4,1]$, and $[1,4,4,4,1]$ (64, 192, and 320 parameters, respectively), while Cheb-KAN uses $M=3$ with $[1,9,1]$, $[1,6,6,1]$, and $[1,6,6,6,1]$ (72, 192, and 336 parameters). Architectures are chosen to match parameter counts across the two models. Cheb-KAN uses soft-penalty weight $\lambda=10^{-1}$. Each entry reports the relative $L_2$ errors (\%) averaged over five random seeds ($\pm$ one standard deviation). Bold indicates the lowest kernel reconstruction error $\mathcal{E}(K)$ among all six configurations at each noise level.}
\label{tab:p1_tab1}

\small
\setlength{\tabcolsep}{5pt}
\begin{tabular}{llcccccc}
\toprule
 & & \multicolumn{2}{c}{Shallow ($L=1$)} & \multicolumn{2}{c}{Medium ($L=2$)} & \multicolumn{2}{c}{Deep ($L=3$)} \\
\cmidrule(lr){3-4}\cmidrule(lr){5-6}\cmidrule(lr){7-8}
Noise & Error
 & \makecell{MC-KAN} & \makecell{Cheb-KAN}
 & \makecell{MC-KAN} & \makecell{Cheb-KAN}
 & \makecell{MC-KAN} & \makecell{Cheb-KAN} \\
\midrule

\multirow{2}{*}{$0.00$}
 & $\mathcal{E}(u)$
 & $0.13 \pm 0.16$
 & $0.08 \pm 0.06$
 & $0.02 \pm 0.01$
 & $0.05 \pm 0.03$
 & $0.07 \pm 0.10$
 & $0.04 \pm 0.03$ \\
 & $\mathcal{E}(K)$
 & $2.77 \pm 3.11$
 & $2.50 \pm 1.35$
 & $\mathbf{0.85 \pm 0.28}$
 & $1.83 \pm 1.09$
 & $2.27 \pm 1.79$
 & $1.95 \pm 0.72$ \\
\addlinespace

\multirow{2}{*}{$0.05$}
 & $\mathcal{E}(u)$
 & $0.66 \pm 0.20$
 & $0.65 \pm 0.16$
 & $0.64 \pm 0.17$
 & $0.67 \pm 0.14$
 & $0.64 \pm 0.16$
 & $0.68 \pm 0.15$ \\
 & $\mathcal{E}(K)$
 & $5.64 \pm 2.57$
 & $4.80 \pm 1.44$
 & $4.71 \pm 2.40$
 & $6.72 \pm 1.81$
 & $\mathbf{4.47 \pm 2.19}$
 & $6.62 \pm 1.99$ \\
\addlinespace

\multirow{2}{*}{$0.10$}
 & $\mathcal{E}(u)$
 & $1.25 \pm 0.35$
 & $1.27 \pm 0.33$
 & $1.25 \pm 0.32$
 & $1.31 \pm 0.28$
 & $1.24 \pm 0.33$
 & $1.31 \pm 0.27$ \\
 & $\mathcal{E}(K)$
 & $7.19 \pm 3.01$
 & $7.26 \pm 1.78$
 & $\mathbf{6.83 \pm 2.96}$
 & $10.05 \pm 3.03$
 & $6.97 \pm 3.02$
 & $11.17 \pm 4.54$ \\
\addlinespace

\multirow{2}{*}{$0.15$}
 & $\mathcal{E}(u)$
 & $1.87 \pm 0.50$
 & $1.89 \pm 0.49$
 & $1.86 \pm 0.49$
 & $1.92 \pm 0.41$
 & $1.85 \pm 0.49$
 & $1.90 \pm 0.48$ \\
 & $\mathcal{E}(K)$
 & $\mathbf{8.78 \pm 3.44}$
 & $10.09 \pm 2.19$
 & $9.04 \pm 3.35$
 & $12.66 \pm 4.03$
 & $9.97 \pm 3.11$
 & $11.64 \pm 2.92$ \\
\bottomrule
\end{tabular}
\end{table}

\begin{figure}[h!]
    \centering
    \includegraphics[width=1\linewidth]{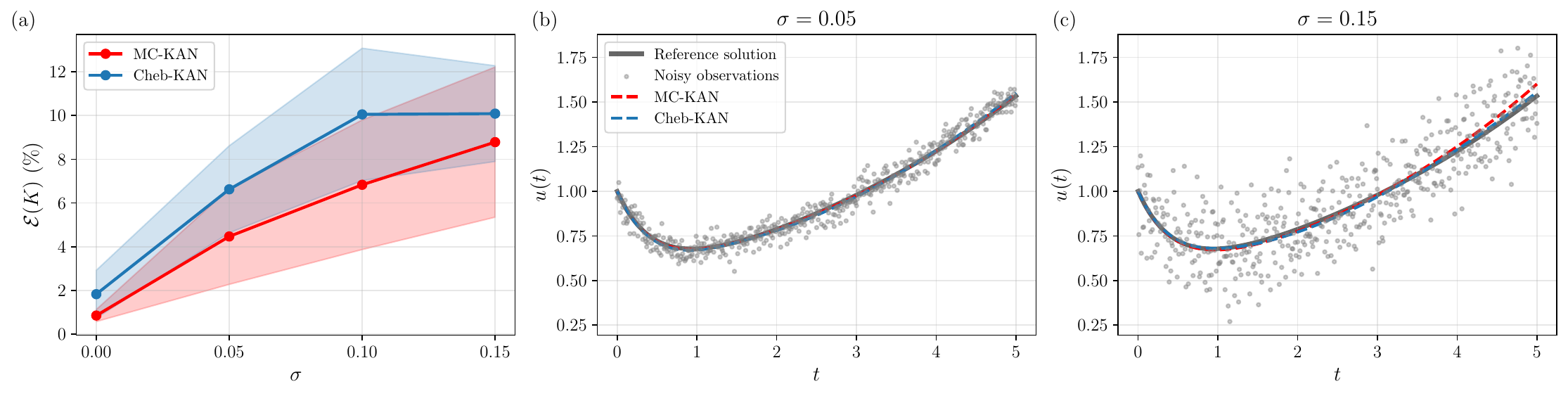}
    \caption{(a) Relative kernel reconstruction error $\mathcal{E}(K)$ (\%) versus noise level for MC-KAN and Cheb-KAN in Experiment~I. At each noise level, the architecture with the lowest mean kernel reconstruction error among all six configurations is selected, and both methods are evaluated using that architecture. Markers denote the mean over five random seeds, while the shaded regions represent $\pm$ one standard deviation. (b)--(c) Solution $u(t)$ for $\sigma=0.05$ and $\sigma=0.15$, respectively. The black solid line denotes the reference solution, gray markers indicate the noisy observations, and the dashed curves show the predictions of MC-KAN(red) and Cheb-KAN(blue).}
    \label{fig:p1-fig1}
\end{figure}

Figure~\ref{fig:p1-fig2} presents the recovered kernels as the pointwise mean over the five random seeds, with the shaded band indicating $\pm$ one standard deviation. The figure corresponds to the same architecture selection used in Fig.~\ref{fig:p1-fig1}. At all noise levels, both methods closely follow the reference exponential kernel. The hard-constrained MC-KAN remains positive, monotonically decreasing, and convex by construction. In the pointwise kernel bands, Cheb-KAN shows a wider spread near the origin and through the middle of the domain, particularly at higher noise levels. We note that the seed-to-seed standard deviation of the error $\mathcal{E}(K)$ in Table~\ref{tab:p1_tab1} is comparable for the two methods in this experiment, so the reduced variability of the hard-constrained parameterization is visible in the pointwise reconstructions rather than in the $\mathcal{E}(K)$ error. The corresponding predicted solution trajectories and optimization loss histories are provided in Appendix (Figures~\ref{fig:appendix-p1-fig1}--\ref{fig:appendix-p1-fig2}).

\begin{figure}[h!]
    \centering
    \includegraphics[width=1\linewidth]{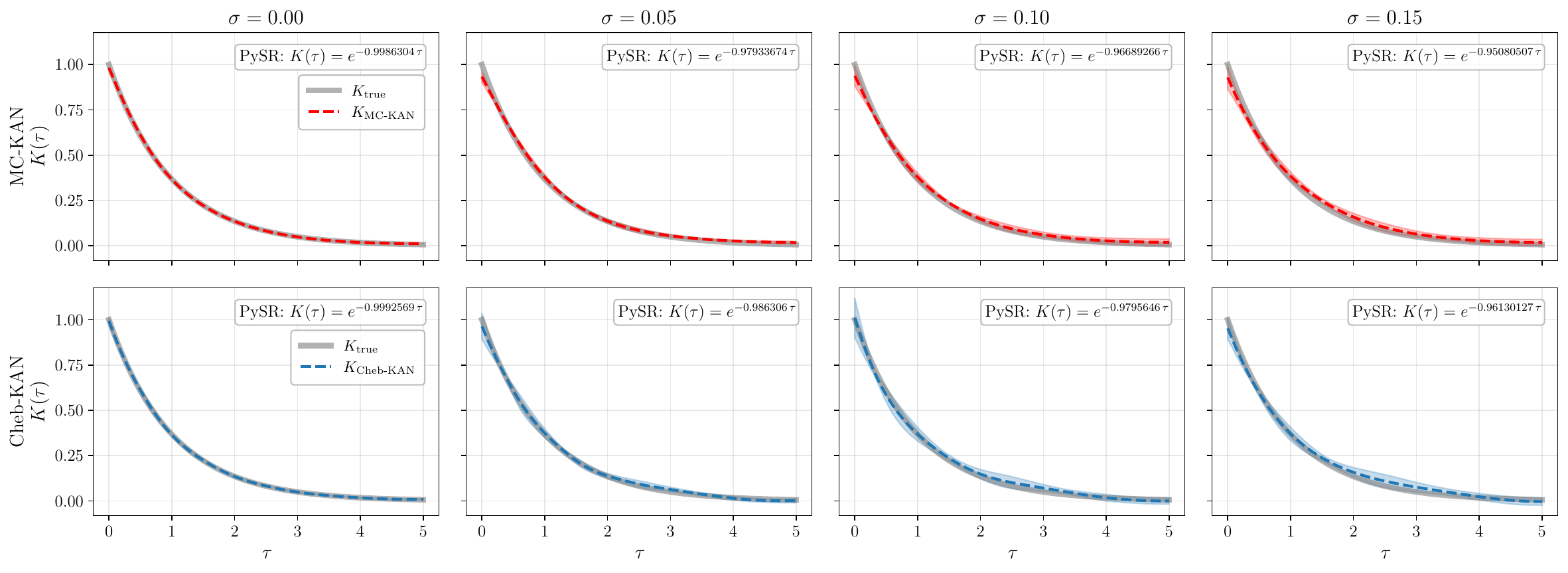}
    \caption{Recovered kernels in Experiment I under increasing noise, showing the best-performing architecture (lowest $\mathcal{E}(K)$) at each noise level. Rows correspond to MC-KAN (top) and Cheb-KAN (bottom), and columns correspond to noise levels $0.00$, $0.05$, $0.10$, $0.15$ (left to right). Each panel shows the true kernel (gray), the pointwise mean reconstruction over five seeds (dashed; MC-KAN in red, Cheb-KAN in blue), and a shaded $\pm$ one standard deviation band. Annotations give the closed-form expression recovered by PySR applied to the mean kernel.}
    \label{fig:p1-fig2}
\end{figure}

To obtain an interpretable closed-form representation, we apply PySR to the mean recovered kernel of each configuration, restricting the operator set to primitives expected for relaxation kernels. The resulting expressions are annotated in Fig.~\ref{fig:p1-fig2}. In all cases, both MC-KAN and Cheb-KAN recover the exponential functional form $ K(\tau)=e^{-r\tau}$, with the learned decay rate remaining close to the true value. At low noise, the recovered rate is approximately $0.99$ for both methods, decreasing to roughly $0.95$--$0.96$ at $\sigma=0.15$. Thus, although MC-KAN provides more accurate numerical kernel reconstructions, both KAN parameterizations successfully identify the correct analytical constitutive law after symbolic regression. 

\subsection{Experiment II: 1D Viscoelastic Wave PIDE}
As a second experiment, we consider the free vibration of a viscoelastic rod, a setting representative of polymer rods and viscoelastic structural elements undergoing free oscillation~\cite{mainardi2022fractional}. We take a homogeneous rod of length $L$, constant mass density $\rho$, and uniform cross-section, fixed at both ends. Neglecting body forces, the balance of linear momentum along the rod axis gives,
\begin{equation}
    \frac{\partial}{\partial x}\sigma(x,t)
    =
    \rho
    \frac{\partial^2}{\partial t^2}u(x,t),
    \label{eq:motion}
\end{equation}
where $u(x,t)$ is the displacement field and $\sigma(x,t)$ the axial
stress. Deformation is measured by the axial strain,
\begin{equation}
    \epsilon(x,t)
    =
    \frac{\partial}{\partial x}u(x,t),
    \label{eq:kinematic}
\end{equation}
and the stress is related to the strain history through the Boltzmann superposition principle, which encodes the fading memory of the viscoelastic material,
\begin{equation}
    \sigma(x,t)
    =
    \epsilon(x,0^+)\,K(t)
    +
    \int_0^t
    K(t-s)\,
    \frac{\partial}{\partial s}
    \epsilon(x,s)\,ds ,
    \label{eq:boltzmann}
\end{equation}
where $K$ is the relaxation kernel and the first term
accounts for the instantaneous elastic response at $t=0^+$. Substituting the constitutive law Eq.~\eqref{eq:boltzmann} and the kinematic relation Eq.~\eqref{eq:kinematic} into the momentum balance Eq.~\eqref{eq:motion}, with the term $\epsilon(x,0^+)K(t)$ omitted because of the zero-displacement initial condition, yields the following PIDE for the displacement,
\begin{equation}
    \rho
    \frac{\partial^2 u}{\partial t^2}(x,t)
    =
    \int_0^t
    K(t-s)\,
    \frac{\partial^3 u}
         {\partial x^2 \partial s}(x,s)\,ds ,
    \qquad (x,t)\in(0,L)\times(0,T],
    \label{eq:pide}
\end{equation}
in which the memory kernel $K(t-s)$ couples the current acceleration to the entire strain-rate history. In this experiment, we set the rod length to $L=1~\mathrm{m}$ and the final time to $T=5~\mathrm{s}$. The rod is fixed at both ends,
\begin{equation}
    u(0,t)=0,
    \qquad
    u(L,t)=0,
    \qquad
    t>0,
    \label{eq:rod-bc}
\end{equation}
and starts from zero displacement with an initial velocity prescribed in the fundamental mode,
\begin{equation}
    u(x,0)=0,
    \qquad
    \frac{\partial u}{\partial t}(x,0)
    =
    \sin\!\left(\frac{\pi x}{L}\right).
    \label{eq:rod-ic}
\end{equation}

To generate reference data, we prescribe the relaxation modulus using the Kohlrausch--Williams--Watts (KWW) stretched exponential kernel,

\begin{equation}
    K(t-s)
    =
    \exp\!\left[
    -\left(
    \frac{t-s}{\tau_0}
    \right)^{\beta}
    \right],
    \label{eq:kww}
\end{equation}
where $\tau_0>0$ is the characteristic relaxation time and $0<\beta\le1$ is the stretching exponent. In this experiment we set $\tau_0=1.5\,\mathrm{s}$ and $\beta=0.5$. The KWW kernel is widely used to model relaxation in glassy polymers, amorphous solids, and other disordered materials, where memory decay cannot be captured by a single exponential time scale~\cite{williams1970non,lindsey1980detailed,metzler2000random}. The exponent $\beta$ sets the degree of non-exponential decay: $\beta=1$ recovers the classical Maxwell exponential kernel, while $\beta<1$ corresponds to a broad distribution of relaxation times characteristic of heterogeneous materials. The choice $\beta=0.5$ therefore represents a more challenging identification problem than the exponential kernel of Experiment~I. 

The reference data are generated by solving Eq.~\eqref{eq:pide} using an implicit Newmark--$\beta$ time-integration scheme together with central finite differences in space and a trapezoidal approximation of the memory integral. The average-acceleration Newmark--$\beta$ scheme ($\gamma=1/2$, $\beta_N=1/4$) is used for time integration. This scheme is classically unconditionally stable for linear elastodynamic
systems, and the chosen spatial and temporal resolutions provide adequate accuracy for the present viscoelastic problem (see Fig.~\ref{fig:appendix-p2-convergence} in Appendix). The resulting displacement field is computed on a grid of $1001$ time steps and $200$ spatial points.

In this experiment, we aim to investigate the effect of data sparsity on kernel discovery. To do so, we construct training datasets by uniformly subsampling both the spatial and temporal dimensions. In all cases the spatial observations are restricted to $50$ uniformly spaced locations ($25\%$ of the spatial grid), while the number of temporal snapshots is varied according to $N_s \in \{101,\,51,\,21,\,11\}$. The resulting dataset is
$ \mathcal{D}
    =
    \left\{
    \bigl(
    t_k,\,
    \mathbf{x}_i,\,
    u(t_k,\mathbf{x}_i)
    \bigr)
    \right\},
$
where $\{\mathbf{x}_i\}_{i=1}^{50}$ denotes the retained spatial locations and $\{t_k\}_{k=1}^{N_s}$ denotes the retained temporal snapshots. The final training datasets contain $10.1\%$, $5.1\%$, $2.1\%$, and $1.1\%$ of the $1001$ reference time steps, corresponding to roughly $2.5\%$, $1.3\%$, $0.5\%$, and $0.3\%$ of the full spatiotemporal grid. No measurement noise is added in this experiment to isolate the effect of data sparsity from observational uncertainty.

As in Experiment~I, the unknown kernel is represented by either MC-KAN or Cheb-KAN. In Cheb-KAN, positivity, monotonic decrease, and convexity are encouraged through equally weighted soft penalty terms with $\lambda=3\times10^{-1}$. The two models use depth-matched architectures with approximately the same number of trainable parameters, and all inputs are linearly normalized. During training, the differentiable solver uses the same spatial grid, time step, and numerical discretization as the reference solver. At each optimization step, the predicted displacement field is computed on the full spatiotemporal grid, and the loss is evaluated only at the observation points in $\mathcal{D}$. Inside the solver, the relaxation kernel is normalized so that $K(0^+)=1$, enforced by dividing the network output by its value at
$\tau_{\min}=10^{-3}$. Fixing the instantaneous modulus keeps the implicit term of the memory operator independent of the kernel, so the effective system matrix is assembled and factored once and reused throughout the forward solve. This normalization is also physically motivated: $K(0^+)$ is the instantaneous glassy modulus, and since only the shape of the relaxation kernel is identifiable from displacement data, fixing it to unity removes the amplitude degeneracy without loss of generality. For evaluation, the learned kernel is sampled at 500 uniformly spaced temporal points, and the kernel reconstruction error is computed with respect to the reference KWW kernel on the same grid. Complete numerical, architectural, and optimization settings are summarized in Table~\ref{tab:exp2-config}.

\begin{table}[h!]
\centering
\small
\setlength{\tabcolsep}{5pt}
\caption{Results for Experiment II, showing the solution and kernel reconstruction errors under temporal sparsity (noise-free) for the 1D viscoelastic wave PIDE. Sparsity is set by the number of time snapshots $N_s$ in the loss, with spatial observations fixed at 50 points. MC-KAN uses polynomial degree $N=6$ with architecture $[1,4,4,1]$ (192 parameters), and Cheb-KAN uses $M=3$ with $[1,6,6,1]$ (192 parameters). Cheb-KAN uses soft-penalty weight $\lambda=3\times10^{-1}$. Each entry reports the relative $L_2$ errors (\%) averaged over five random seeds ($\pm$ one standard deviation). Bold indicates the lowest kernel reconstruction error $\mathcal{E}(K)$.}
\label{tab:p2_sparsity}
\begin{tabular}{llcc}
\toprule
Sparsity ($N_s$) & Error & MC-KAN & Cheb-KAN \\
\midrule
\multirow{2}{*}{$101$}
 & $\mathcal{E}(u)$ & $0.403 \pm 0.133$ & $0.384 \pm 0.077$ \\
 & $\mathcal{E}(K)$ & $\mathbf{1.233 \pm 0.089}$ & $2.141 \pm 0.166$ \\
\addlinespace

\multirow{2}{*}{$51$}
 & $\mathcal{E}(u)$ & $0.402 \pm 0.136$ & $0.444 \pm 0.080$ \\
 & $\mathcal{E}(K)$ & $\mathbf{1.233 \pm 0.089}$ & $2.056 \pm 0.209$ \\
\addlinespace
\multirow{2}{*}{$21$}
 & $\mathcal{E}(u)$ & $0.405 \pm 0.143$ & $0.589 \pm 0.373$ \\
 & $\mathcal{E}(K)$ & $\mathbf{1.234 \pm 0.093}$ & $2.254 \pm 0.338$ \\
\addlinespace
\multirow{2}{*}{$11$}
 & $\mathcal{E}(u)$ & $0.383 \pm 0.134$ & $0.378 \pm 0.156$ \\
 & $\mathcal{E}(K)$ & $\mathbf{1.204 \pm 0.086}$ & $1.938 \pm 0.413$ \\
\bottomrule
\end{tabular}
\end{table}

\begin{figure}[h!]
    \centering
    \includegraphics[width=1\linewidth]{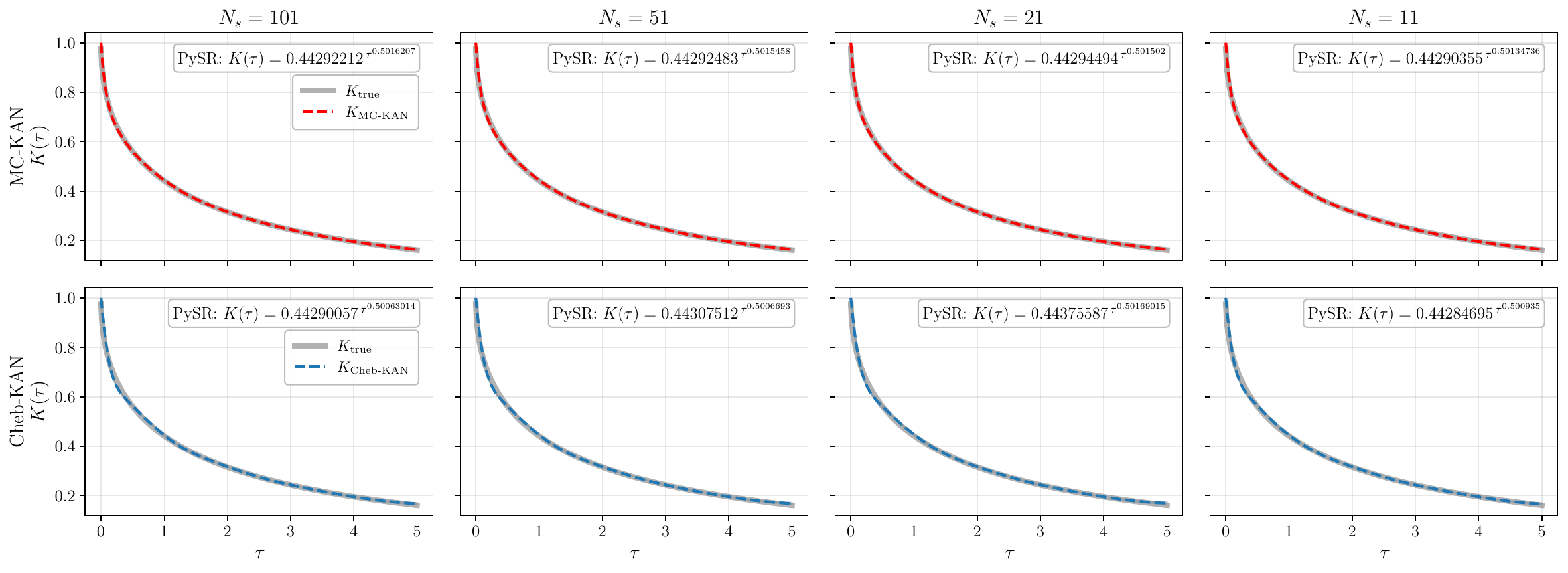}
    \caption{Recovered kernels in Experiment II under temporal sparsity (noise-free). Rows correspond to MC-KAN (top) and Cheb-KAN (bottom), and columns correspond to decreasing snapshot counts $N_s = 101, 51, 21, 11$ (left to right). Each panel shows the true kernel (gray), the pointwise mean reconstruction over five seeds (dashed; MC-KAN in red, Cheb-KAN in blue), and a shaded $\pm$ one standard deviation band. Annotations give the closed-form expression recovered by PySR applied to the mean kernel. }
    \label{fig:p2-fig1}
\end{figure}

\begin{table}[h!]
\centering
\caption{Results for Experiment II, showing the stretched-exponential parameters recovered from the MC-KAN and Cheb-KAN kernel reconstructions across sparsity levels. PySR is applied independently to each kernel reconstruction for all sparsity levels and random seeds. The direct PySR output has the form ($K(\tau)=a^{\tau^\beta}$), which is rewritten as a stretched exponential using ($\tau_0=\left(-\ln a\right)^{-1/\beta}$). Each entry reports the mean ($\pm$) one standard deviation over the five PySR fits corresponding to the five independent runs.}
\label{tab:p2_kww_params}
\small
\setlength{\tabcolsep}{5pt}
\begin{tabular}{llcccc}
\toprule
 & & \multicolumn{2}{c}{$\tau_0$ (true $=1.5$)} & \multicolumn{2}{c}{$\beta$ (true $=0.5$)} \\
\cmidrule(lr){3-4}\cmidrule(lr){5-6}
$N_s$ & Method & Recovered & Err.\ (\%) & Recovered & Err.\ (\%) \\
\midrule
\multirow{2}{*}{$101$}
 & MC-KAN   & $1.5059 \pm 0.0010$ & $0.39$ & $0.5016 \pm 0.0019$ & $0.32$ \\
 & Cheb-KAN & $1.5069 \pm 0.0025$ & $0.46$ & $0.5006 \pm 0.0012$ & $0.12$ \\
\addlinespace
\multirow{2}{*}{$51$}
 & MC-KAN   & $1.5060 \pm 0.0011$ & $0.40$ & $0.5016 \pm 0.0019$ & $0.32$ \\
 & Cheb-KAN & $1.5083 \pm 0.0020$ & $0.55$ & $0.5007 \pm 0.0012$ & $0.14$ \\
\addlinespace
\multirow{2}{*}{$21$}
 & MC-KAN   & $1.5062 \pm 0.0015$ & $0.41$ & $0.5015 \pm 0.0018$ & $0.30$ \\
 & Cheb-KAN & $1.5127 \pm 0.0106$ & $0.85$ & $0.5017 \pm 0.0016$ & $0.34$ \\
\addlinespace
\multirow{2}{*}{$11$}
 & MC-KAN   & $1.5061 \pm 0.0019$ & $0.41$ & $0.5014 \pm 0.0016$ & $0.28$ \\
 & Cheb-KAN & $1.5061 \pm 0.0034$ & $0.41$ & $0.5010 \pm 0.0011$ & $0.20$ \\
\bottomrule
\end{tabular}
\end{table}

Our results show that sparsity has little effect on kernel identification in this experiment. For each method the kernel error is nearly constant as the number of snapshots is reduced from $101$ to $11$ (Table~\ref{tab:p2_sparsity}), and the solution errors remain near $0.4\%$ throughout. The recovered KWW parameters are likewise stable, staying within $1\%$ of the true values at every sparsity level, as demonstrated in Table~\ref{tab:p2_kww_params}. The stretched-exponential kernel thus remains identifiable even when the temporal data barely resolve the dynamics. The corresponding sparse solution reconstruction errors and optimization loss histories are provided in Appendix (Figs.~\ref{fig:p2-appendix-fig1} and\ref{fig:p2-appendix-fig2}).

Although both methods are insensitive to temporal sparsity, MC-KAN produces lower kernel reconstruction errors than Cheb-KAN at every sparsity level, approximately $1.2\%$ compared with $2.0\%$, corresponding to a reduction of roughly $40\%$. In addition, the variability across random initializations is consistently smaller for MC-KAN. These trends are reflected in Fig.~\ref{fig:p2-fig1}, shown as the pointwise mean over seeds with a $\pm$ one standard deviation band. The MC-KAN reconstructions remain positive, monotone, and convex by construction, and they track the true kernel with a narrow band across all sparsity levels. The Cheb-KAN reconstructions also visually fit the true kernel well, but slight traces of oscillatory behavior are visible between $\tau=0$ and $\tau=1$. 

Despite the differences in reconstruction accuracy, symbolic regression discovers the same constitutive law for both KAN parameterizations. As shown in Fig.~\ref{fig:p2-fig1}, the recovered expression of the form $0.443^{\,\tau^{\beta}}$ can be rewritten as $e^{-(\tau/\tau_0)^{\beta}}$, and recalculating its parameters gives values of $\tau_0$ and $\beta$ reported in Table~\ref{tab:p2_kww_params}, in close agreement with the true $\tau_0=1.5$ and $\beta=0.5$. Across all sparsity levels, the recovered $\tau_0$ and $\beta$ deviate from their true values by less than $1\%$ for both methods.

\subsection{Experiment III: 2D Nonlocal Reaction-Diffusion Equation}

As a third experiment, we consider a two-dimensional nonlocal reaction-diffusion equation. Unlike Experiments~I and~II, which focus on temporal memory kernels, this problem involves the discovery of a
spatial interaction kernel governing nonlocal interactions between points in the domain. This model arises in phase separation, pattern formation, and collective dynamics, where the evolution at a given location depends on the surrounding state through a spatial convolution operator~\cite{bates1997traveling, bates2009numerical}. We study an attractive, distance-decaying interaction~\cite{topaz2006nonlocal, fetecau2011swarm}, for which the shape constraints are exact properties of the kernel. The experiment tests whether enforcing them by construction is advantageous when the true kernel belongs to that class.
 
The governing equation is,

\begin{equation}
    \frac{\partial u}{\partial t}(\mathbf{x},t)
    =
    D\nabla^2 u(\mathbf{x},t)
    +
    \alpha
    \int_{\Omega}
    K(x-\xi,y-\eta)
    \bigl(
    u(\xi,\eta,t)-u(x,y,t)
    \bigr)
    \,d\xi\,d\eta
    +
    u(\mathbf{x},t)
    -
    u^3(\mathbf{x},t),
    \label{eq:nonlocal-ac}
\end{equation}
for $(\mathbf{x},t)\in\Omega\times(0,T]$, where $\Omega\subset\mathbb{R}^2$, $\mathbf{x}=(x,y)$, $u(\mathbf{x},t)$ is the state variable, $D>0$ is the diffusion coefficient, $\alpha>0$ controls the strength of the nonlocal interactions, and $K(x-\xi,y-\eta)$ is the unknown spatial interaction kernel. Throughout this experiment, we set $D=0.001$, $\alpha=0.12$, $L_x=L_y=10$, and the final time to $T=6~\mathrm{s}$. Periodic boundary conditions are imposed in both spatial directions, and the initial condition consists of a smooth random perturbation of the homogeneous state (\ref{appendix:p3-ic}). To generate reference data, we prescribe the interaction kernel,
\begin{equation}
    K(x-\xi,y-\eta)
    =
    \exp\!\Bigg(
    -
    \left|
    \frac{x-\xi}{\tau_x}
    \right|^{\beta_x}
    -
    \left|
    \frac{y-\eta}{\tau_y}
    \right|^{\beta_y}
    -
    \delta
    \left|
    \frac{x-\xi}{\tau_x}
    \right|^{\gamma_x}
    \left|
    \frac{y-\eta}{\tau_y}
    \right|^{\gamma_y}
    \Bigg),
    \label{eq:spatial_kernel}
\end{equation}
where $(x,y)$ denotes the target location and
$(\xi,\eta)$ denotes the source location contributing to the nonlocal interaction. Since the problem is on a periodic domain, the kernel is represented numerically as a function of the periodic displacement ($x-\xi,y-\eta$) sampled on the periodic displacement grid. The resulting nonlocal convolution operator is evaluated as a circular convolution using the FFT. The parameters $\tau_x$ and $\tau_y$ determine the characteristic interaction lengths in the horizontal and vertical directions, respectively, while $\beta_x$ and $\beta_y$ control the decay rate of the kernel. The mixed term weighted by $\delta$ introduces anisotropic coupling between the two spatial directions. Throughout this experiment, we use,
\[
\tau_x=3.0,
\qquad
\tau_y=1.0,
\qquad
\beta_x=0.4,
\qquad
\beta_y=0.8,
\]
\[
\gamma_x=0.5,
\qquad
\gamma_y=0.5,
\qquad
\delta=0.5.
\]

We selected these parameters deliberately to introduce anisotropy through differing interaction lengths ($\tau_x \neq \tau_y$) and decay rates ($\beta_x \neq \beta_y$). The resulting kernel is therefore substantially more challenging to identify than an isotropic one, providing a more demanding benchmark for multidimensional kernel recovery.

The reference data are generated by solving Eq.~\eqref{eq:nonlocal-ac} on a periodic domain using a Fourier spectral discretization in space together with a semi-implicit time-stepping scheme (IMEX). The nonlocal interaction integral is evaluated as a circular convolution on the periodic domain using the FFT, reducing the computational complexity of the interaction term (see Fig.~\ref{fig:appendix-p3-convergence}). The solution is computed on a $256\times256$ spatial grid at $121$ time points. In this experiment, we aim to mimic realistic measurements by introducing both sparsity and additive noise. The full-resolution solution is subsampled in space and time. Observations are restricted to a $32\times32$ spatial grid and taken at the $7$ time points, $ t =\{0,\,1,\,2,\,3,\,4,\,5,\,6\},$ giving the dataset 
$
\mathcal{D}
=
\left\{
\bigl(
t_k,\,
x_i,\,
y_j,\,
u(t_k,x_i,y_j)
\bigr)
\right\},
k=1,\ldots,7, i,j=1,\ldots,32 .
$
Additive Gaussian noise is then applied,
\begin{equation}
u^{\mathrm{obs}}
=
u
+
\sigma\varepsilon,
\qquad
\varepsilon
\sim
\mathcal{N}(0,1),
\end{equation}
where $\sigma\in\{0.00,0.02,0.04,0.06,0.08,0.10,0.15\}$. As in the previous experiments, we consider two KAN parameterizations for the unknown kernel: the hard-constrained MC-KAN and the soft-constrained Cheb-KAN. In Cheb-KAN, positivity, monotonic decrease, and convexity are encouraged through equally weighted soft penalty terms with $\lambda=10^{-2}$. The two models use depth-matched architectures with approximately the same number of trainable parameters. Here, both KANs take two inputs corresponding to the spatial displacements in the horizontal and vertical directions. Because the interaction kernel is even in each coordinate, the networks are parameterized in terms of the absolute displacements $\tau_1=|x-\xi|$ and $\tau_2=|y-\eta|$. For MC-KAN, the shape constraints are enforced independently along each lag direction. For both models, the inputs are linearly normalized. During training, the differentiable solver uses the same spatial grid, time step, and numerical discretization as the reference solver. At each optimization step, the full solution field is computed on the $256\times256$ grid and evaluated against the observations in $\mathcal{D}$. The loss is the MSE over the spatial observation points at each time snapshot, averaged over the snapshots. For evaluation, the learned spatial kernel is sampled on the same $256\times256$ lag grid as the reference kernel, and the kernel reconstruction error is computed pointwise on this common grid. Complete numerical, architectural, and optimization settings are summarized in Table~\ref{tab:exp3-config}.

Table~\ref{tab:p3-tab1} reports the solution and kernel reconstruction errors across noise levels for the linear normalization. In the noise-free case, Cheb-KAN is marginally lower ($2.64\%$ versus $3.64\%$). However, at every noisy level ($\sigma \ge 0.02$), MC-KAN achieves a lower kernel error $\mathcal{E}(K)$ than Cheb-KAN, and the gap widens with noise, reaching nearly a factor of two at $\sigma = 0.15$ ($12.36\%$ versus $21.95\%$). The solution error $\mathcal{E}(u)$ remains small for both
methods and increases smoothly with noise. The two methods are comparable at low noise, with MC-KAN increasingly favored as noise grows. The solution error is an order of magnitude smaller than the kernel error,  because the forward solution is less sensitive to small variations in the kernel than the inverse problem itself. The corresponding final-time solution reconstruction errors are shown in Appendix (Fig.~\ref{fig:p3-appendix-fig1}).

\begin{table}[h!]
\centering
\small
\caption{Results for Experiment III, showing the solution and kernel reconstruction errors for the 2D nonlocal reaction-diffusion problem under additive Gaussian noise, with $7$ snapshots on a $32\times32$ spatial grid. MC-KAN uses polynomial degree $N=6$ with architecture $[2,4,4,1]$ (224 parameters), and Cheb-KAN uses $M=3$ with $[2,6,6,1]$ (216 parameters). Cheb-KAN uses soft-penalty weight $\lambda=10^{-2}$. Both models use linear input normalization. Each entry reports the relative $L_2$ errors (\%) averaged over five noise realizations (fixed initialization; $\sigma=0.00$ is a single deterministic run), and $\pm$ is one standard deviation. Bold marks the lower $\mathcal{E}(K)$.}
\label{tab:p3-tab1}
\resizebox{\textwidth}{!}{%
\begin{tabular}{llccccccc}
\toprule
Error & Method & $\sigma=0.00$ & $\sigma=0.02$ & $\sigma=0.04$ & $\sigma=0.06$ & $\sigma=0.08$ & $\sigma=0.10$ & $\sigma=0.15$ \\
\midrule
\multirow{2}{*}{$\mathcal{E}(u)$}
& MC-KAN (Linear)
& $0.143$
& $0.308 \pm 0.060$
& $0.455 \pm 0.095$
& $0.630 \pm 0.122$
& $0.776 \pm 0.160$
& $0.925 \pm 0.205$
& $1.31 \pm 0.226$ \\
& Cheb-KAN
& $0.057$
& $0.403 \pm 0.036$
& $0.693 \pm 0.067$
& $0.978 \pm 0.093$
& $1.27 \pm 0.152$
& $1.54 \pm 0.161$
& $2.16 \pm 0.248$ \\
\addlinespace
\multirow{2}{*}{$\mathcal{E}(K)$}
& MC-KAN (Linear)
& $3.64$
& $\mathbf{5.29 \pm 0.78}$
& $\mathbf{6.40 \pm 1.43}$
& $\mathbf{7.70 \pm 2.41}$
& $\mathbf{8.81 \pm 3.14}$
& $\mathbf{9.76 \pm 3.72}$
& $\mathbf{12.36 \pm 4.11}$ \\
& Cheb-KAN
& $\mathbf{2.64}$
& $6.48 \pm 1.30$
& $8.87 \pm 1.98$
& $11.28 \pm 2.51$
& $14.36 \pm 2.82$
& $16.38 \pm 3.34$
& $21.95 \pm 4.08$ \\
\bottomrule
\end{tabular}%
}
\end{table}

Figure~\ref{fig:p3-fig1} shows the true kernel together with the pointwise absolute error of each method at $\sigma=0.00$ and $\sigma=0.08$. At zero noise both
methods reconstruct the kernel accurately, with errors confined to a narrow region near the origin. Under noise the difference between the two becomes pronounced:
the MC-KAN error remains concentrated near the axes, whereas the Cheb-KAN error spreads into broad off-axis regions where the soft penalty fails to enforce the
correct decay. In both methods, however, and even at zero noise, the error is concentrated along the lines $\tau_1=0$ and $\tau_2=0$, where the kernel exhibits
a sharp gradient near the origin. During training, the origin is excluded to represent general kernels that may exhibit singular behavior at $\tau=0$.

\begin{figure}[h!]
    \centering
    \includegraphics[width=0.8\linewidth]{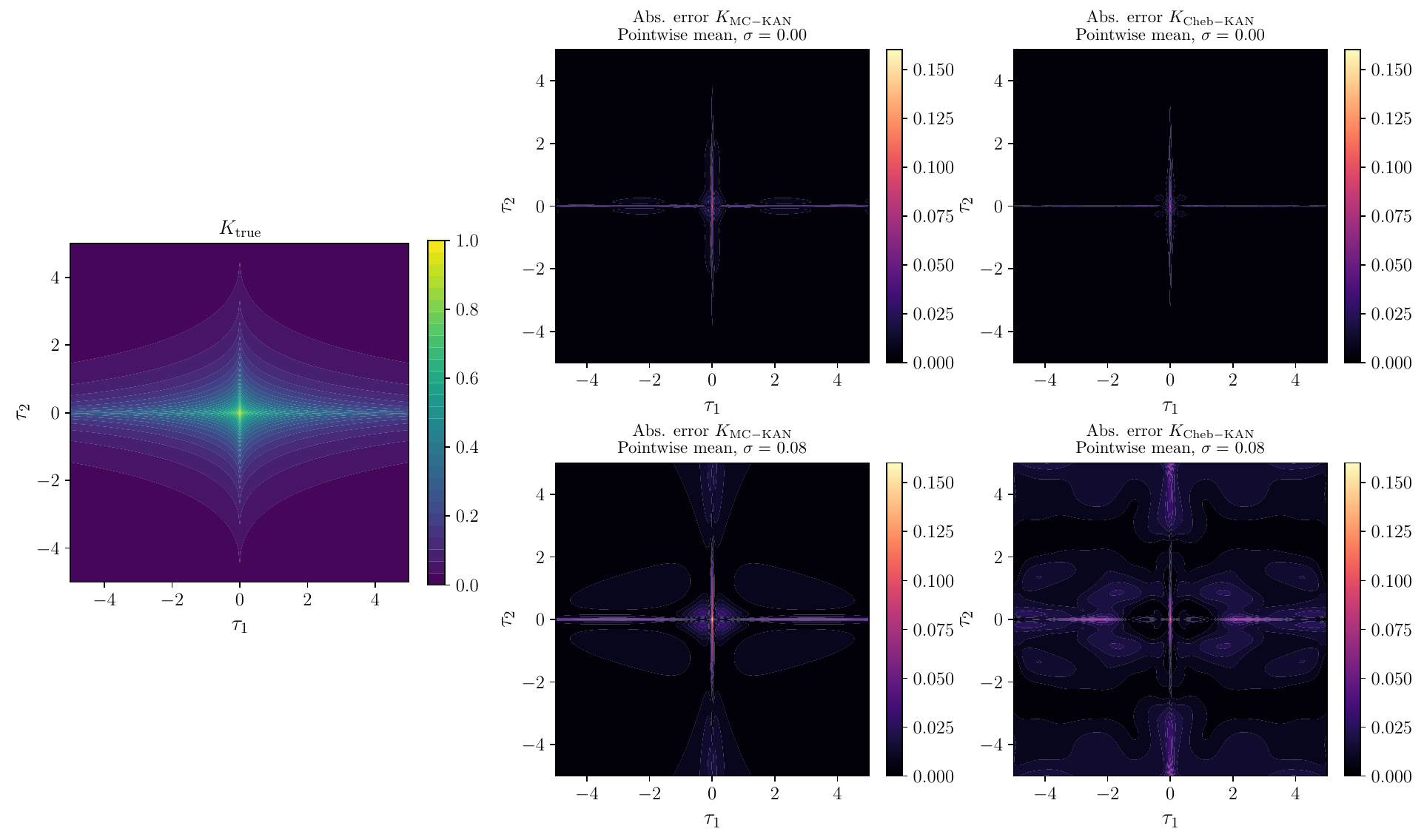}
    \caption{Results for Experiment III, showing the true kernel $K_{\mathrm{true}}$ (left) and the pointwise absolute error of the reconstructed kernel for MC-KAN and Cheb-KAN at $\sigma=0.00$ (top) and $\sigma=0.08$ (bottom). Error maps are averaged over five noise realizations with fixed initialization ($\sigma=0.00$ is a single deterministic run).}
    \label{fig:p3-fig1}
\end{figure}

To examine this behavior, in Fig.~\ref{fig:p3-fig2} we plot the recovered kernel along the two axis slices, $K(\tau_1,\tau_2{\approx}0)$ and $K(\tau_1{\approx}0,\tau_2)$. For consistency with the training setup, the one-dimensional slices are evaluated at $\tau_2=10^{-3}$ and $\tau_1=10^{-3}$, respectively. The true kernel rises steeply toward the origin, a consequence of the stretched exponents $\beta_i<1$, and neither method fully resolves this sharp gradient with
the linear input normalization. MC-KAN follows the true profile more closely than Cheb-KAN, but it too underestimates the peak near $\tau_i\approx0$. 

\begin{figure}[h!]
    \centering
    \includegraphics[width=0.7\linewidth]{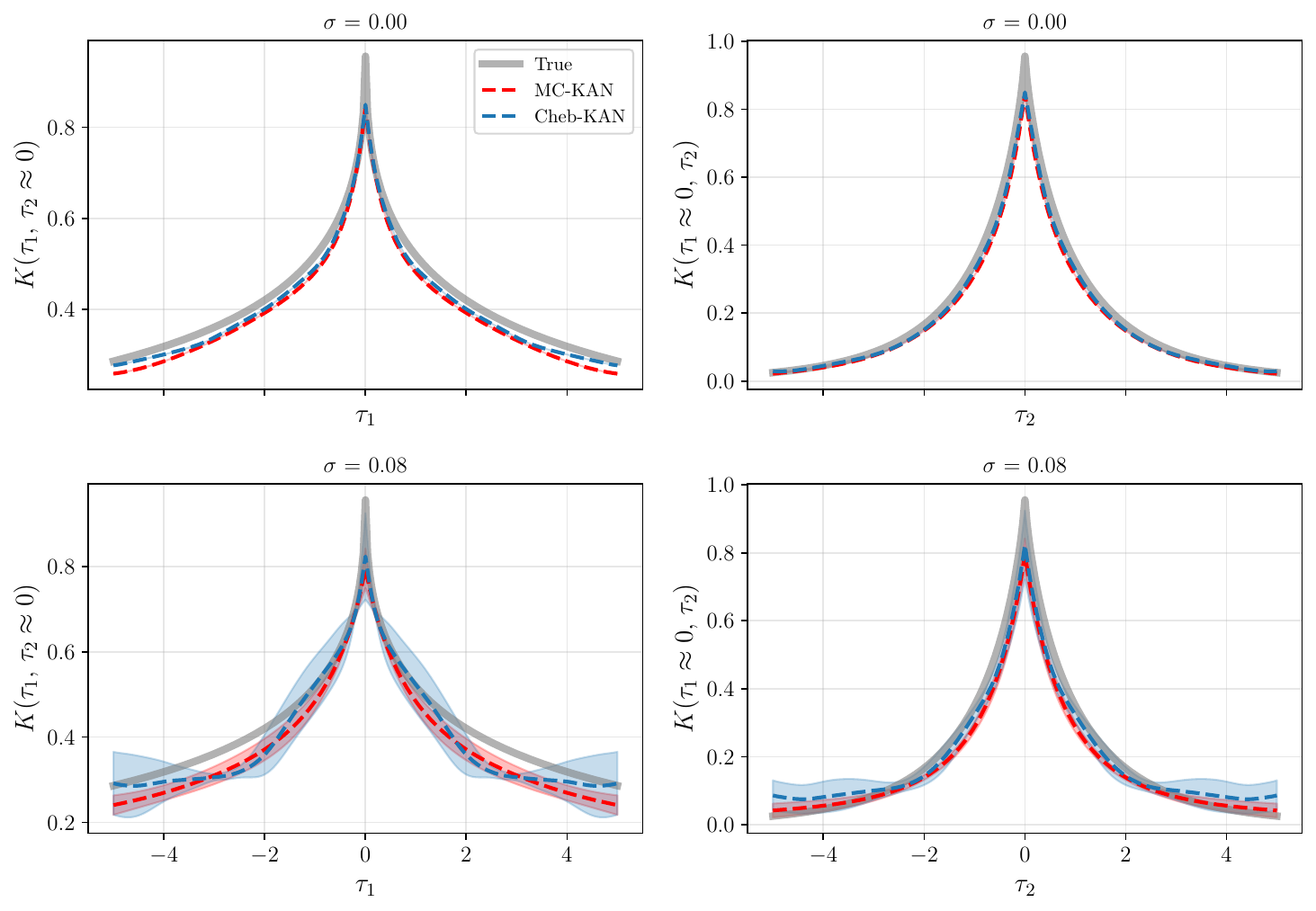}
    \caption{Results for Experiment III, showing the recovered kernel along the axis slices $K(\tau_1,\tau_2{\approx}0)$ (left) and $K(\tau_1{\approx}0,\tau_2)$ (right) at $\sigma=0.00$ (top) and $\sigma=0.08$ (bottom).  Each panel shows the true kernel (gray), the pointwise mean reconstruction over five noise realizations with fixed initialization (dashed; MC-KAN in red, Cheb-KAN in blue), and a shaded $\pm$ one standard deviation band. Both models use linear input normalization, and the training dataset consist of 7 temporal snapshots on a $32\times32$ spatial grid.}
    \label{fig:p3-fig2}
\end{figure}

This difficulty motivates a change of input representation. The linear normalization maps the lag uniformly, allocating limited resolution to the
near-origin region where the kernel varies most rapidly. We therefore replace it with a power-law normalization that concentrates resolution near the origin. Figure~\ref{fig:p3-fig3}(a) compares the kernel error across noise levels for the linear normalization, a learnable power-law exponent, and a fixed exponent. The learnable variant improves the reconstruction at low noise, where the steep gradient can be resolved, but its advantage diminishes as noise increases and the exponent becomes harder to identify reliably. We therefore first optimize the exponent in the noise-free setting, obtaining $\alpha_1=\alpha_2=0.75$, and then fix this value for all subsequent noisy experiments. This fixed normalization retains most of the benefit of the learnable normalization while providing a single normalization that remains stable across all noise levels.

\begin{figure}[h!]
    \centering
    \includegraphics[width=1\linewidth]{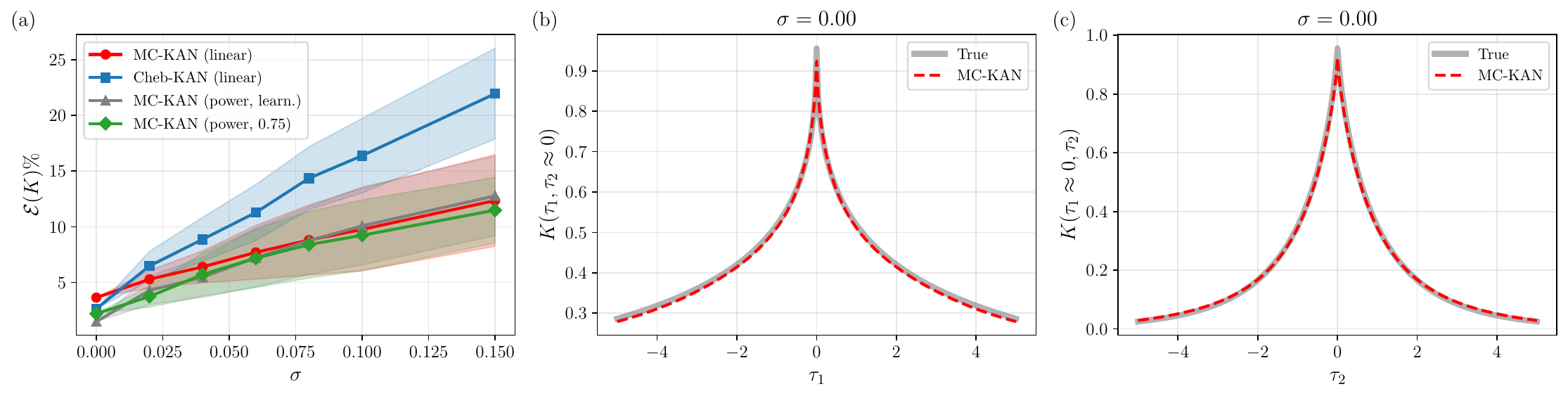}
    \caption{(a) Relative kernel error $\mathcal{E}(K)$ versus noise for three MC-KAN input normalizations (linear, learnable power-law, and fixed $\alpha_1=\alpha_2=0.75$) and the Cheb-KAN baseline. The training data consist of seven temporal snapshots on a $32\times32$ spatial grid. Markers denote the mean over five noise realizations with fixed initialization, and shaded bands indicate $\pm$ one standard deviation. (b)--(c) Recovered kernel along the axis slices $K(\tau_1,\tau_2{\approx}0)$ and $K(\tau_1{\approx}0,\tau_2)$, respectively, for the model trained using $25$ temporal snapshots on a $256\times256$ spatial grid with learnable power-law input normalization at $\sigma=0.00$.}
    \label{fig:p3-fig3}
\end{figure}

Finally, we increase the temporal sampling from $7$ to $25$ snapshots, holding the spatial sparsity fixed, to assess how much of the remaining error is data-limited. Table~\ref{tab:p3-tab2} reports the results for the fixed power-law MC-KAN against the Cheb-KAN baseline. The corresponding optimization loss histories are provided in Appendix (Fig.~\ref{fig:p3-appendix-fig2}). The additional data reduce the kernel error substantially for both methods, and MC-KAN again achieves the lower error at every noise level. For example, at $\sigma=0.15$ its kernel error falls to $8.03\%$, roughly half that of Cheb-KAN ($14.55\%$) and well below the $12.36\%$ obtained with seven snapshots. Despite the increased temporal resolution, a residual error remains near the sharp central peak (Fig.~\ref{fig:p3-fig4}), likely because temporal information alone is insufficient to fully resolve this localized feature when the observations remain spatially sparse. To determine whether this residual error is intrinsic to the inverse problem or a consequence of the data sparsity, we repeat the experiment using dense $256\times256$ spatial observations with $25$ temporal snapshots. As shown in Fig.~\ref{fig:p3-fig3}(b)--(c), the reconstructed kernel now closely matches the true peak, meaning that the previous mismatch was primarily caused by the limited spatial sampling rather than by the representational capacity of the network or non-identifiability of the inverse problem. The dense observations provide sufficient information to resolve the steep spatial gradients near the origin.

\begin{table}[h!]
\centering
\small
\caption{Results for Experiment III, showing the solution and kernel reconstruction errors for the 2D nonlocal reaction-diffusion problem under additive Gaussian noise, with $25$ snapshots on a $32\times32$ spatial grid. MC-KAN uses polynomial degree $N=6$ with architecture $[2,4,4,1]$ (224 parameters), and Cheb-KAN uses $M=3$ with $[2,6,6,1]$ (216 parameters). Cheb-KAN uses soft-penalty weight $\lambda=10^{-2}$. MC-KAN uses power-fixed normalization with exponent $0.75$, and Cheb-KAN uses linear input normalization. Each entry reports the relative $L_2$ errors (\%) averaged over five noise realizations (fixed initialization; $\sigma=0.00$ is a single deterministic run), and $\pm$ is one standard deviation. Bold marks the lower $\mathcal{E}(K)$.}
\label{tab:p3-tab2}
\resizebox{\textwidth}{!}{%
\begin{tabular}{llccccccc}
\toprule
Error & Method & $\sigma=0.00$ & $\sigma=0.02$ & $\sigma=0.04$ & $\sigma=0.06$ & $\sigma=0.08$ & $\sigma=0.10$ & $\sigma=0.15$ \\
\midrule
\multirow{2}{*}{$\mathcal{E}(u)$}
& MC-KAN (Power, $0.75$)
& $0.0941$
& $0.176 \pm 0.0332$
& $0.296 \pm 0.0649$
& $0.413 \pm 0.0891$
& $0.531 \pm 0.112$
& $0.652 \pm 0.139$
& $0.937 \pm 0.184$ \\
& Cheb-KAN
& $0.0596$
& $0.279 \pm 0.0346$
& $0.468 \pm 0.0589$
& $0.651 \pm 0.101$
& $0.823 \pm 0.132$
& $1.01 \pm 0.175$
& $1.40 \pm 0.270$ \\
\addlinespace
\multirow{2}{*}{$\mathcal{E}(K)$}
& MC-KAN (Power, $0.75$)
& $\mathbf{2.00}$
& $\mathbf{2.44 \pm 0.31}$
& $\mathbf{3.31 \pm 0.89}$
& $\mathbf{4.53 \pm 1.85}$
& $\mathbf{5.41 \pm 2.14}$
& $\mathbf{6.24 \pm 2.19}$
& $\mathbf{8.03 \pm 2.03}$ \\
& Cheb-KAN
& $2.59$
& $5.39 \pm 0.60$
& $6.46 \pm 1.11$
& $8.14 \pm 1.90$
& $9.73 \pm 2.59$
& $11.29 \pm 3.03$
& $14.55 \pm 4.47$ \\
\bottomrule
\end{tabular}%
}
\end{table}

\begin{figure}[h!]
    \centering
    \includegraphics[width=0.7\linewidth]{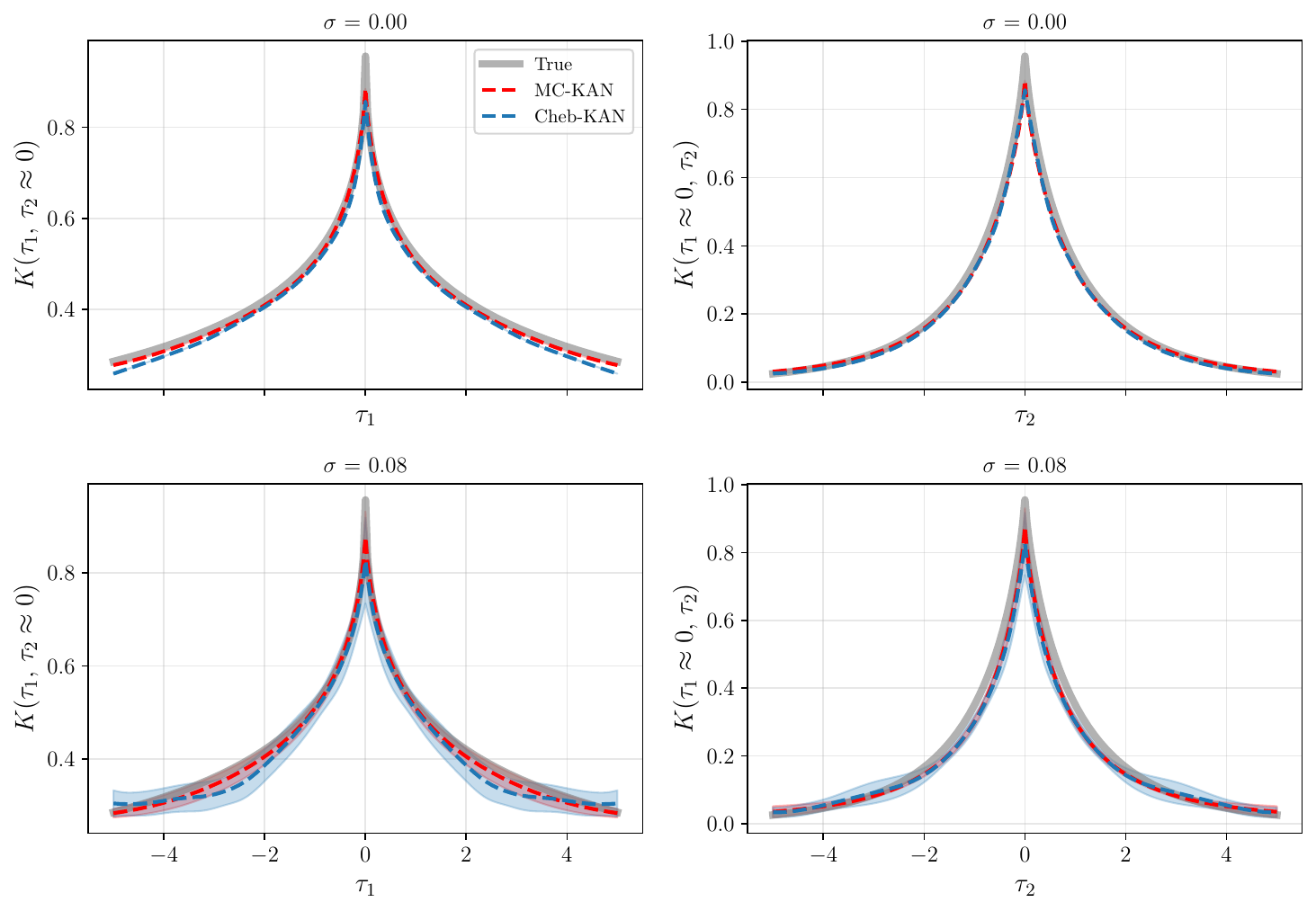} 
    \caption{Results for Experiment III, showing the recovered kernel along the axis slices $K(\tau_1,\tau_2{\approx}0)$ (left) and $K(\tau_1{\approx}0,\tau_2)$ (right) at $\sigma=0.00$ (top) and $\sigma=0.08$ (bottom).  Each panel shows the true kernel (gray), the pointwise mean reconstruction over five noise realizations with fixed initialization (dashed; MC-KAN in red, Cheb-KAN in blue), and a shaded $\pm$ one standard deviation band. MC-KAN uses power-fixed normalization with exponent $0.75$, and Cheb-KAN uses linear input normalization. The training dataset consist of 25 temporal snapshots on a $32\times32$ spatial grid.}
    \label{fig:p3-fig4}
\end{figure}

The next step is to apply symbolic regression to recover a closed-form expression for the kernel from the optimized network. Unlike Experiments I and II, where the true kernel is one-dimensional, here it is a two-dimensional anisotropic function with a coupling term between the two spatial lags. Directly recovering such a multivariate coupled form is beyond the reach of symbolic regression in this setting, as the search space grows rapidly with input dimension. We therefore apply PySR to one-dimensional slices of the learned kernel, obtained by fixing one lag coordinate, and report the recovered expressions for Figs.~\ref{fig:p3-fig1},~\ref{fig:p3-fig2}, and \ref{fig:p3-fig4} in Table~\ref{tab:p3-tab3}. The table lists the expressions returned directly by PySR and, where possible, their rewriting in the $\exp(-c\,\tau^{p})$ form in the last column. The $K(\tau_1,\tau_2\approx0)$ slice is recovered more accurately than $K(\tau_1\approx 0,\tau_2)$ in all cases, probably because the fast-decaying $K(\tau_1\approx 0,\tau_2)$ slice is near-zero over most of the lag range, concentrating the fitting information near the origin. In contrast, the slowly-decaying $K(\tau_1,\tau_2\approx0)$ slice varies across the whole domain. At zero noise with $7$ snapshots, both models recover a physically correct stretched-exponential kernel. At $\sigma=0.08$,
MC-KAN still recovers the $K(\tau_1,\tau_2\approx0)$ slice in the correct exponential form, while Cheb-KAN returns a non-admissible power-law form for both slices despite comparable $L_2$ error. Increasing the number of snapshots to $25$ results in exponential recoveries in more cases and lower errors overall. Notably, the MC-KAN $K(\tau_1,\tau_2\approx0)$ error is almost unchanged between $\sigma=0$ and $\sigma=0.08$, indicating that the hard-constrained recovery is robust to noise for the well-resolved slice.

The two slices capture only the separable part of the kernel. To recover
the coupling term, we compute the residual between the full learned kernel and its separable slices, then apply symbolic regression to it. Because the recovered slices are exponential in form, we analyze the kernel in log space, where $\log K$ decomposes into the two marginal contributions and a residual,
\begin{equation}
    \log K(\tau_1,\tau_2)
    = \log K(\tau_1,\tau_2\approx0) + \log K(\tau_1\approx0,\tau_2) + K_{\mathrm{res}}(\tau_1,\tau_2),
\end{equation}
with $K_{\mathrm{res}} = -\delta\,C$. For our best case, MC-KAN with $25$ snapshots at zero noise, PySR recovers $\delta\,C_{\mathrm{rec}} = 0.156\,\tau_1^{0.481}\,\tau_2$, compared with the true coupling $\delta\,C = 0.289\,\tau_1^{0.5}\,\tau_2^{0.5}$, with relative $L_2$ error $0.20$. The recovered term reproduces the product form and the $\tau_1$ exponent ($0.48$ versus $0.5$), but underestimates the coupling coefficient ($0.156$ versus $0.289$) and recovers the $\tau_2$ dependence as approximately linear rather than $\tau_2^{0.5}$.

\begin{table}[h!]
\centering
\caption{Results for Experiment III, showing the closed-form kernel slices recovered by PySR. Several of the direct PySR expressions admit an equivalent stretched-exponential form, reported in the "Exp.\ form" column. Free constants below 0.020 (0.012 and 0.019) are dropped in this conversion. Expressions that do not reduce to this form are marked with a dash. $\mathcal{E}(K_{\text{slice}})$ denotes the relative $L_2$ error of the recovered expression against the true kernel slice.}
\label{tab:p3-tab3}
\small
\setlength{\tabcolsep}{5pt}
\begin{tabular}{lll l c l}
\toprule
Model & $\sigma$ & Slice & Recovered expression (PySR) & $\mathcal{E}(K_{\text{slice}})$ & Exp.\ form \\
\midrule
\multirow{2}{*}{True kernel}
 & \multirow{2}{*}{--}
   & $K(\tau_1,0)$ &  & & $\exp\!\big(-0.644\,\tau_1^{0.4}\big)$ \\
 & & $K(0,\tau_2)$ &  & & $\exp\!\big(-\tau_2^{0.8}\big)$ \\
\midrule
\multirow{4}{*}{MC-KAN linear (7 snaps)}
 & \multirow{2}{*}{$0$}
   & $K(\tau_1,0)$ & $0.487^{\,\tau_1^{0.393}}$ & 0.0835 & $\exp\!\big(-0.721\,\tau_1^{0.393}\big)$ \\
 & & $K(0,\tau_2)$ & $0.333^{\,(\tau_2+0.081)^{0.759}}$ & 0.1428 & $\exp\!\big(-1.099\,(\tau_2+0.081)^{0.759}\big)$ \\
\cmidrule(lr){2-6}
 & \multirow{2}{*}{$0.08$}
   & $K(\tau_1,0)$ & $0.481^{\,\tau_1^{0.421}}$ & 0.1136 & $\exp\!\big(-0.733\,\tau_1^{0.421}\big)$ \\
 & & $K(0,\tau_2)$ & $1.308^{\,\exp(0.512^{\tau_2})}-1.279$ & 0.1920 & $-$ \\
\midrule
\multirow{4}{*}{Cheb-KAN (7 snaps)}
 & \multirow{2}{*}{$0$}
   & $K(\tau_1,0)$ & $0.496^{\,\tau_1^{0.389}}$ & 0.0593 & $\exp\!\big(-0.701\,\tau_1^{0.389}\big)$ \\
 & & $K(0,\tau_2)$ & $0.316^{\,\tau_2^{0.703}}$ & 0.1112 & $\exp\!\big(-1.153\,\tau_2^{0.703}\big)$ \\
\cmidrule(lr){2-6}
 & \multirow{2}{*}{$0.08$}
   & $K(\tau_1,0)$ & $(\tau_1+1.416)^{-0.763}$ & 0.0897 & $-$ \\
 & & $K(0,\tau_2)$ & $(\tau_2+1.113)^{-1.590}$ & 0.1603 & $-$ \\
\midrule
\multirow{4}{*}{\makecell[l]{MC-KAN power 0.75\\(25 snaps)}}
 & \multirow{2}{*}{$0$}
   & $K(\tau_1,0)$ & $0.510^{\,\tau_1^{0.403}}$ & 0.0365 & $\exp\!\big(-0.673\,\tau_1^{0.403}\big)$ \\
 & & $K(0,\tau_2)$ & $0.364^{\,\tau_2^{0.823}}\!\cdot 0.871 + 0.012$ & 0.1055 & $0.871\,\exp\!\big(-1.011\,\tau_2^{0.823}\big)$ \\
\cmidrule(lr){2-6}
 & \multirow{2}{*}{$0.08$}
   & $K(\tau_1,0)$ & $0.511^{\,\tau_1^{0.411}}$ & 0.0394 & $\exp\!\big(-0.671\,\tau_1^{0.411}\big)$ \\
 & & $K(0,\tau_2)$ & $0.304^{\,(\tau_2+0.079)^{0.756}}+0.019$ & 0.1537 & $\exp\!\big(-1.191\,(\tau_2+0.079)^{0.756}\big)$ \\
\midrule
\multirow{4}{*}{Cheb-KAN (25 snaps)}
 & \multirow{2}{*}{$0$}
   & $K(\tau_1,0)$ & $0.497^{\,\tau_1^{0.392}}$ & 0.0591 & $\exp\!\big(-0.699\,\tau_1^{0.392}\big)$ \\
 & & $K(0,\tau_2)$ & $0.318^{\,\tau_2^{0.713}}$ & 0.1052 & $\exp\!\big(-1.145\,\tau_2^{0.713}\big)$ \\
\cmidrule(lr){2-6}
 & \multirow{2}{*}{$0.08$}
   & $K(\tau_1,0)$ & $0.287\,\exp\!\big(0.544^{\tau_1}\big)$ & 0.0640 & $-$ \\
 & & $K(0,\tau_2)$ & $0.293^{\,\tau_2^{0.622}}$ & 0.1715 & $\exp\!\big(-1.228\,\tau_2^{0.622}\big)$ \\
\bottomrule
\end{tabular}
\end{table}

\section{Conclusion}\label{sec:conclusion}
In this study, we developed a differentiable-solver-based framework for  discovering memory and nonlocal kernels in integro-differential equations  (IDEs) directly from spatiotemporal observations.  Within this framework, we investigated two KAN parameterizations for the unknown kernel: a hard-constrained Bernstein-based Monotone--Convex KAN (MC-KAN), in which positivity, monotonic decrease, and convexity are enforced by construction through constraints on the Bernstein coefficients, and a soft-constrained Chebyshev-based KAN (Cheb-KAN), in which the same physical properties are encouraged through penalty terms in the loss function. We systematically evaluated both approaches across benchmark IDEs of  increasing complexity, from a Volterra equation to a viscoelastic wave 
PIDE and a two-dimensional nonlocal reaction--diffusion equation, under sparse and noisy observation settings. Our experiments demonstrate that both KAN parameterizations provide an effective framework for kernel identification. When the kernel family is relatively simple, as in the one-dimensional Volterra equation and the viscoelastic wave PIDE, both approaches recover accurate kernel representations and symbolic closed-form expressions,  with MC-KAN achieving the lowest kernel reconstruction error in most configurations of Experiment~I and at every sparsity level of Experiment~II. The advantage of hard constraint enforcement becomes more pronounced for the sparse, noisy, and multidimensional nonlocal reaction-diffusion problem. In this setting, the hard-constrained MC-KAN produces more accurate kernel reconstructions and suppresses the nonphysical oscillations that remain possible under soft penalty enforcement. For the most challenging configuration considered, consisting of $7$ temporal snapshots on a $32\times32$ spatial grid with $\sigma=0.15$ observation noise, MC-KAN with linear input normalization reduced the kernel reconstruction error by roughly $44\%$ relative to Cheb-KAN. These results indicate that the main benefit of hard architectural constraints is not necessarily improved performance for simpler,  better-observed inverse problems, but a more robust and physically  consistent kernel representation as the observations become increasingly sparse and noisy.

The proposed framework has several limitations. First, the current hard-constrained MC-KAN is designed for kernels that are positive, monotonically decreasing, and convex, and therefore is not directly applicable to oscillatory, sign-changing, or non-monotone kernels.  Second, although symbolic regression recovered accurate expressions for  the one-dimensional benchmarks, the two-dimensional example required a  staged procedure in which the principal-axis contributions and the  coupling residual were recovered separately. The resulting coupled expression reproduced the product structure only approximately and did not fully recover the coefficient and exponent of the true interaction term.
Future work will therefore focus on extending the framework to broader classes of physical kernels, developing symbolic discovery methods capable of recovering fully coupled multidimensional kernel functions, and validating the proposed approach on experimental datasets.

\color{black}
\section{Acknowledgements}
S.A.F. acknowledges support from the U.S. National Science Foundation under the Collaborations in Artificial Intelligence and Geosciences (CAIG, Award No. 2530611). This research was also supported in part by funding from the Wilkes Center for Climate Science \& Policy (Award No. 01-00055-6000-44773) at the University of Utah.

\section{Conflict of Interest}
The authors declare no conflict of interests.

\section{Code Availability}
The source code will be made publicly available upon publication of this work.
\def\mybibdoicolor{\color{black}}
\newcommand*{\doi}[1]{\href{\detokenize{#1}} {\raggedright\mybibdoicolor{DOI: \detokenize{#1}}}}
\bibliographystyle{unsrtnat}
\bibliography{references.bib}

@book{mainardi2022fractional,
  title={Fractional calculus and waves in linear viscoelasticity: an introduction to mathematical models},
  author={Mainardi, Francesco},
  year={2022},
  publisher={World Scientific}
}

@article{williams1970non,
  title={Non-symmetrical dielectric relaxation behaviour arising from a simple empirical decay function},
  author={Williams, Graham and Watts, David C},
  journal={Transactions of the Faraday society},
  volume={66},
  pages={80--85},
  year={1970},
  publisher={Royal Society of Chemistry}
}

@article{lindsey1980detailed,
  title={Detailed comparison of the Williams--Watts and Cole--Davidson functions},
  author={Lindsey, Christopher P and Patterson, Gary D},
  journal={The Journal of chemical physics},
  volume={73},
  number={7},
  pages={3348--3357},
  year={1980},
  publisher={AIP Publishing}
}

@inproceedings{zappala2023neural,
  title={Neural integro-differential equations},
  author={Zappala, Emanuele and Fonseca, Antonio H de O and Moberly, Andrew H and Higley, Michael J and Abdallah, Chadi and Cardin, Jessica A and van Dijk, David},
  booktitle={Proceedings of the aaai conference on artificial intelligence},
  volume={37},
  number={9},
  pages={11104--11112},
  year={2023}
}

@article{zappala2024learning,
  title={Learning integral operators via neural integral equations},
  author={Zappala, Emanuele and Fonseca, Antonio Henrique de Oliveira and Caro, Josue Ortega and Moberly, Andrew Henry and Higley, Michael James and Cardin, Jessica and Dijk, David van},
  journal={Nature Machine Intelligence},
  volume={6},
  number={9},
  pages={1046--1062},
  year={2024},
  publisher={Nature Publishing Group UK London}
}

@article{liu1989limited,
  title={On the limited memory BFGS method for large scale optimization},
  author={Liu, Dong C and Nocedal, Jorge},
  journal={Mathematical programming},
  volume={45},
  number={1},
  pages={503--528},
  year={1989},
  publisher={Springer}
}

@article{caputo2008diffusion,
  title={Diffusion with memory in two cases of biological interest},
  author={Caputo, Michele and Cametti, Cesare},
  journal={Journal of theoretical biology},
  volume={254},
  number={3},
  pages={697--703},
  year={2008},
  publisher={Elsevier}
}

@article{bates1997traveling,
  title={Traveling waves in a convolution model for phase transitions},
  author={Bates, Peter W and Fife, Paul C and Ren, Xiaofeng and Wang, Xuefeng},
  journal={Archive for Rational Mechanics and Analysis},
  volume={138},
  number={2},
  pages={105--136},
  year={1997},
  publisher={Springer}
}

@article{bates2009numerical,
  title={Numerical analysis for a nonlocal Allen-Cahn equation},
  author={Bates, Peter W and Brown, Sarah and Han, Jianlong},
  journal={International Journal of Numerical Analysis and Modeling},
  volume={6},
  number={1},
  pages={33--49},
  year={2009}
}

@article{topaz2006nonlocal,
  title={A nonlocal continuum model for biological aggregation},
  author={Topaz, Chad M and Bertozzi, Andrea L and Lewis, Mark A},
  journal={Bulletin of mathematical biology},
  volume={68},
  number={7},
  pages={1601--1623},
  year={2006},
  publisher={Springer}
}

@article{fetecau2011swarm,
  title={Swarm dynamics and equilibria for a nonlocal aggregation model},
  author={Fetecau, Razvan C and Huang, Yanghong and Kolokolnikov, Theodore},
  journal={Nonlinearity},
  volume={24},
  number={10},
  pages={2681--2716},
  year={2011}
}

@book{cushing2013integrodifferential,
  title={Integrodifferential equations and delay models in population dynamics},
  author={Cushing, Jim M},
  year={2013},
  publisher={Springer Science \& Business Media}
}

@book{brunner2017volterra,
  title={Volterra integral equations: an introduction to theory and applications},
  author={Brunner, Hermann},
  volume={30},
  year={2017},
  publisher={Cambridge University Press}
}

@article{nunziato1971heat,
  title={On heat conduction in materials with memory},
  author={Nunziato, Jace W},
  journal={Quarterly of Applied Mathematics},
  volume={29},
  number={2},
  pages={187--204},
  year={1971}
}

@article{dehghan2006solution,
  title={Solution of a partial integro-differential equation arising from viscoelasticity},
  author={Dehghan, Mehdi},
  journal={International Journal of Computer Mathematics},
  volume={83},
  number={1},
  pages={123--129},
  year={2006},
  publisher={Taylor \& Francis}
}

@article{faroughi2024physics,
  title={Physics-guided, physics-informed, and physics-encoded neural networks and operators in scientific computing: Fluid and solid mechanics},
  author={Faroughi, Salah A and Pawar, Nikhil M and Fernandes, Celio and Raissi, Maziar and Das, Subasish and Kalantari, Nima K and Kourosh Mahjour, Seyed},
  journal={Journal of Computing and Information Science in Engineering},
  volume={24},
  number={4},
  pages={040802},
  year={2024},
  publisher={American Society of Mechanical Engineers}
}

@article{dafermos1970asymptotic,
  title={Asymptotic stability in viscoelasticity},
  author={Dafermos, Constantine M},
  journal={Archive for rational mechanics and analysis},
  volume={37},
  number={4},
  pages={297--308},
  year={1970},
  publisher={Springer}
}

@article{metzler2000random,
  title={The random walk's guide to anomalous diffusion: a fractional dynamics approach},
  author={Metzler, Ralf and Klafter, Joseph},
  journal={Physics reports},
  volume={339},
  number={1},
  pages={1--77},
  year={2000},
  publisher={Elsevier}
}

@article{metzler2014anomalous,
  title={Anomalous diffusion models and their properties: non-stationarity, non-ergodicity, and ageing at the centenary of single particle tracking},
  author={Metzler, Ralf and Jeon, Jae-Hyung and Cherstvy, Andrey G and Barkai, Eli},
  journal={Physical Chemistry Chemical Physics},
  volume={16},
  number={44},
  pages={24128--24164},
  year={2014},
  publisher={Royal Society of Chemistry}
}

@article{gurtin1968general,
  title={A general theory of heat conduction with finite wave speeds},
  author={Gurtin, Morton E and Pipkin, Allen C},
  journal={Archive for Rational Mechanics and Analysis},
  volume={31},
  number={2},
  pages={113--126},
  year={1968},
  publisher={Springer}
}

@article{schadle2006fast,
  title={Fast and oblivious convolution quadrature},
  author={Sch{\"a}dle, Achim and L{\'o}pez-Fern{\'a}ndez, Mar{\'\i}a and Lubich, Christian},
  journal={SIAM Journal on Scientific Computing},
  volume={28},
  number={2},
  pages={421--438},
  year={2006},
  publisher={SIAM}
}

@article{gao2022kernel,
  title={A kernel-independent sum-of-exponentials method},
  author={Gao, Zixuan and Liang, Jiuyang and Xu, Zhenli},
  journal={Journal of Scientific Computing},
  volume={93},
  number={2},
  pages={40},
  year={2022},
  publisher={Springer}
}

@article{chorin2000optimal,
  title={Optimal prediction and the Mori--Zwanzig representation of irreversible processes},
  author={Chorin, Alexandre J and Hald, Ole H and Kupferman, Raz},
  journal={Proceedings of the National Academy of Sciences},
  volume={97},
  number={7},
  pages={2968--2973},
  year={2000},
  publisher={The National Academy of Sciences}
}

@article{gouasmi2017priori,
  title={A priori estimation of memory effects in reduced-order models of nonlinear systems using the Mori--Zwanzig formalism},
  author={Gouasmi, Ayoub and Parish, Eric J and Duraisamy, Karthik},
  journal={Proceedings. Mathematical, Physical, and Engineering Sciences},
  volume={473},
  number={2205},
  pages={20170385},
  year={2017}
}

@article{brunton2016sindy,
  author  = {Brunton, Steven L. and Proctor, Joshua L. and Kutz, J. Nathan},
  title   = {Discovering governing equations from data by sparse identification of nonlinear dynamical systems},
  journal = {Proceedings of the National Academy of Sciences},
  volume  = {113},
  number  = {15},
  pages   = {3932--3937},
  year    = {2016},
  doi     = {10.1073/pnas.1517384113},
}

@article{rudy2017pdefind,
  author  = {Rudy, Samuel H. and Brunton, Steven L. and Proctor, Joshua L. and Kutz, J. Nathan},
  title   = {Data-driven discovery of partial differential equations},
  journal = {Science Advances},
  volume  = {3},
  number  = {4},
  pages   = {e1602614},
  year    = {2017},
  doi     = {10.1126/sciadv.1602614},
}

@article{lu2021deeponet,
  author  = {Lu, Lu and Jin, Pengzhan and Pang, Guofei and Zhang, Zhongqiang and Karniadakis, George Em},
  title   = {Learning nonlinear operators via {DeepONet} based on the universal approximation theorem of operators},
  journal = {Nature Machine Intelligence},
  volume  = {3},
  number  = {3},
  pages   = {218--229},
  year    = {2021},
  doi     = {10.1038/s42256-021-00302-5},
}

@article{faroughi2026symbolic,
  title={Symbolic--KAN: Kolmogorov-Arnold Networks with Discrete Symbolic Structure for Interpretable Learning},
  author={Faroughi, Salah A and Mostajeran, Farinaz and Arzani, Amirhossein and Faroughi, Shirko},
  journal={arXiv preprint arXiv:2603.23854},
  year={2026}
}

@article{chen2018neural,
  title={Neural ordinary differential equations},
  author={Chen, Ricky TQ and Rubanova, Yulia and Bettencourt, Jesse and Duvenaud, David K},
  journal={Advances in neural information processing systems},
  volume={31},
  year={2018}
}

@article{onken2020discretize,
  author  = {Onken, Derek and Ruthotto, Lars},
  title   = {Discretize-Optimize vs. Optimize-Discretize for Time-Series Regression and Continuous Normalizing Flows},
  journal = {arXiv preprint arXiv:2005.13420},
  year    = {2020},
}

@inproceedings{gholami2019anode,
  author    = {Gholami, Amir and Keutzer, Kurt and Biros, George},
  title     = {{ANODE}: Unconditionally Accurate Memory-Efficient Gradients for Neural {ODE}s},
  booktitle = {Proceedings of the Twenty-Eighth International Joint Conference on Artificial Intelligence (IJCAI)},
  year      = {2019},
}

@book{engl1996regularization,
  title={Regularization of inverse problems},
  author={Engl, Heinz Werner and Hanke, Martin and Neubauer, Andreas},
  volume={375},
  year={1996},
  publisher={Springer Science \& Business Media}
}

@article{soussou1970application,
  title={Application of Prony series to linear viscoelasticity},
  author={Soussou, JE and Moavenzadeh, F and Gradowczyk, MH},
  journal={Transactions of the Society of Rheology},
  volume={14},
  number={4},
  pages={573--584},
  year={1970},
  publisher={The Society of Rheology}
}

@article{hanyga2018simple,
  title={A simple proof of a duality theorem with applications in scalar and anisotropic viscoelasticity},
  author={Hanyga, Andrzej},
  journal={arXiv preprint arXiv:1805.07275},
  year={2018}
}

@article{luchko2020complete,
  title={On complete monotonicity of solution to the fractional relaxation equation with the n th level fractional derivative},
  author={Luchko, Yuri},
  journal={Mathematics},
  volume={8},
  number={9},
  pages={1561},
  year={2020},
  publisher={MDPI}
}

@article{hanyga2013wave,
  title={Wave propagation in linear viscoelastic media with completely monotonic relaxation moduli},
  author={Hanyga, Andrzej},
  journal={Wave Motion},
  volume={50},
  number={5},
  pages={909--928},
  year={2013},
  publisher={Elsevier}
}

@article{farouki2012bernstein,
  title={The Bernstein polynomial basis: A centennial retrospective},
  author={Farouki, Rida T},
  journal={Computer Aided Geometric Design},
  volume={29},
  number={6},
  pages={379--419},
  year={2012},
  publisher={Elsevier}
}

@article{koenig2024kan,
  title={KAN-ODEs: Kolmogorov--Arnold network ordinary differential equations for learning dynamical systems and hidden physics},
  author={Koenig, Benjamin C and Kim, Suyong and Deng, Sili},
  journal={Computer Methods in Applied Mechanics and Engineering},
  volume={432},
  pages={117397},
  year={2024},
  publisher={Elsevier}
}

@article{faroughi2026kolmogorov,
  title={Kolmogorov-Arnold networks for data-driven, physics-informed, and deep-operator learning: a review, synthesis, and new analysis},
  author={Faroughi, Salah A and Mostajeran, Farinaz and Mashhadzadeh, Amin Hamed and Faroughi, Shirko},
  journal={Neural networks},
  pages={108791},
  year={2026},
  publisher={Elsevier}
}

@article{cranmer2023pysr,
  author  = {Cranmer, Miles},
  title   = {Interpretable Machine Learning for Science with {PySR} and {SymbolicRegression.jl}},
  journal = {arXiv preprint arXiv:2305.01582},
  year    = {2023},
  doi     = {10.48550/arXiv.2305.01582},
}

@article{janno2000inverse,
  title={Inverse problems for memory kernels by Laplace transform methods},
  author={Janno, Jaan and von Wolfersdorf, Lothar},
  journal={Zeitschrift f{\"u}r Analysis und ihre Anwendungen},
  volume={19},
  number={2},
  pages={489--510},
  year={2000}
}

@article{thakolkaran2025ickan,
  author  = {Thakolkaran, Prakash and Guo, Yaqi and Saini, Shivam and Peirlinck, Mathias and Alheit, Benjamin and Kumar, Siddhant},
  title   = {Can {KAN} {CAN}s? {Input}-convex {Kolmogorov}-{Arnold} networks ({KAN}s) as hyperelastic constitutive artificial neural networks ({CAN}s)},
  journal = {Computer Methods in Applied Mechanics and Engineering},
  volume  = {443},
  pages   = {118089},
  year    = {2025},
  doi     = {10.1016/j.cma.2025.118089},
}

@article{deschatre2026ickan,
  author  = {Deschatre, Thomas and Warin, Xavier},
  title   = {Input Convex Kolmogorov Arnold Networks},
  journal = {arXiv preprint arXiv:2505.21208},
  year    = {2026},
}

@inproceedings{amos2017input,
  title={Input convex neural networks},
  author={Amos, Brandon and Xu, Lei and Kolter, J Zico},
  booktitle={International conference on machine learning},
  pages={146--155},
  year={2017},
  organization={PMLR}
}

@article{difonzo2025physics,
  title={Physics informed neural networks for an inverse problem in peridynamic models},
  author={Difonzo, Fabio V and Lopez, Luciano and Pellegrino, Sabrina F},
  journal={Engineering with Computers},
  volume={41},
  number={6},
  pages={4003--4012},
  year={2025},
  publisher={Springer}
}

@article{khaldi2026learning,
  title={Learning memory kernels in second-order Volterra models: a data-driven approach for viscoelastic systems},
  author={Khaldi, Yacine and Belaifa, Meriem and Benzaoui, Amir},
  journal={Mechanics of Time-Dependent Materials},
  volume={30},
  number={2},
  pages={42},
  year={2026},
  publisher={Springer}
}

@article{breda2025sparse,
  title={Sparse identification of delay equations with distributed memory},
  author={Breda, Dimitri and Tanveer, Muhammad and Wu, Jianhong},
  journal={arXiv preprint arXiv:2512.21070},
  year={2025}
}

@article{carrillo2025sparse,
  title={Sparse identification of nonlocal interaction kernels in nonlinear gradient flow equations via partial inversion},
  author={Carrillo, Jose A and Estrada-Rodriguez, Gissell and Mikolas, Laszlo and Tang, Sui},
  journal={Mathematical Models and Methods in Applied Sciences},
  volume={35},
  number={05},
  pages={1073--1131},
  year={2025},
  publisher={World Scientific}
}

@article{yu6404003physics,
  title={Physics-Guided Symbolic Regression of Nonlocal Transport Memory from Sparse Observations in Aquatic Systems},
  author={Yu, Xiangnan and Zhang, Yong and Sun, HongGuang and Chen, Zhibo and Chen, Yuntian},
  year={2026},
  journal={Preprint, available at SSRN 6404003}
}

@article{pandolfi2017identification,
  title={Identification of the relaxation kernel in diffusion processes and viscoelasticity with memory via deconvolution},
  author={Pandolfi, Luciano},
  journal={Mathematical Methods in the Applied Sciences},
  volume={40},
  number={7},
  pages={2542--2549},
  year={2017},
  publisher={Wiley Online Library}
}

@inproceedings{durdiev2024kernel,
  title={Kernel determination problem for a integro--differential heat equation with a variable thermal conductivity},
  author={Durdiev, Durdimurod and Nuriddinov, Javlon},
  booktitle={AIP Conference Proceedings},
  volume={3004},
  number={1},
  pages={040014},
  year={2024},
  organization={AIP Publishing LLC}
}

@article{pandolfi2015identification,
  title={Identification of a relaxation kernel using two boundary measures},
  author={Pandolfi, Luciano},
  journal={arXiv preprint arXiv:1503.03883},
  year={2015}
}

@article{durdiev2022memory,
  title={Memory kernel reconstruction problems in the integro-differential equation of rigid heat conductor},
  author={Durdiev, Durdimurod K and Zhumaev, Zhonibek Zh},
  journal={Mathematical Methods in the Applied Sciences},
  volume={45},
  number={14},
  pages={8374--8388},
  year={2022},
  publisher={Wiley Online Library}
}

@article{cavaterra1994identifying,
  title={Identifying memory kernels in linear thermoviscoelasticity of Boltzmann type},
  author={Cavaterra, Cecilia and Grasselli, Maurizio},
  journal={Mathematical Models and Methods in Applied Sciences},
  volume={4},
  number={06},
  pages={807--842},
  year={1994},
  publisher={World Scientific}
}

@incollection{lamm2000survey,
  title={A survey of regularization methods for first-kind Volterra equations},
  author={Lamm, Patricia K},
  booktitle={Surveys on solution methods for inverse problems},
  pages={53--82},
  year={2000},
  publisher={Springer}
}

@article{mostajeran2024epi,
  title={Epi-ckans: Elasto-plasticity informed kolmogorov-arnold networks using chebyshev polynomials},
  author={Mostajeran, Farinaz and Faroughi, Salah A},
  journal={arXiv preprint arXiv:2410.10897},
  year={2024}
}

@article{mostajeran2025scaled,
  title={Scaled-cpikans: Spatial variable and residual scaling in Chebyshev-based physics-informed Kolmogorov-Arnold networks},
  author={Mostajeran, Farinaz and Faroughi, Salah A},
  journal={Journal of Computational Physics},
  volume={537},
  pages={114116},
  year={2025},
  publisher={Elsevier}
}

@article{ss2024chebyshev,
  title={Chebyshev polynomial-based kolmogorov-arnold networks: An efficient architecture for nonlinear function approximation},
  author={SS, Sidharth and AR, Keerthana and KP, Anas and others},
  journal={arXiv preprint arXiv:2405.07200},
  year={2024}
}

@article{mostajeran2025minpo,
  title={MINPO: Memory-Informed Neural Pseudo-Operator to Resolve Nonlocal Spatiotemporal Dynamics},
  author={Mostajeran, Farinaz and Tleubek, Aruzhan and Faroughi, Salah A},
  journal={arXiv preprint arXiv:2512.17273},
  year={2025}
}

@article{kingma2014adam,
  title={Adam: A method for stochastic optimization},
  author={Kingma, Diederik P and Ba, Jimmy},
  journal={arXiv preprint arXiv:1412.6980},
  year={2014}
}

@article{paszke2019pytorch,
  title={Pytorch: An imperative style, high-performance deep learning library},
  author={Paszke, Adam and Gross, Sam and Massa, Francisco and Lerer, Adam and Bradbury, James and Chanan, Gregory and Killeen, Trevor and Lin, Zeming and Gimelshein, Natalia and Antiga, Luca and others},
  journal={Advances in neural information processing systems},
  volume={32},
  year={2019}
}

@article{baydin2018automatic,
  title={Automatic differentiation in machine learning: a survey},
  author={Baydin, Atilim Gunes and Pearlmutter, Barak A and Radul, Alexey Andreyevich and Siskind, Jeffrey Mark},
  journal={Journal of machine learning research},
  volume={18},
  number={153},
  pages={1--43},
  year={2018}
}

@article{more1994line,
  title={Line search algorithms with guaranteed sufficient decrease},
  author={Mor{\'e}, Jorge J and Thuente, David J},
  journal={ACM Transactions on Mathematical Software (TOMS)},
  volume={20},
  number={3},
  pages={286--307},
  year={1994},
  publisher={ACM New York, NY, USA}
}

@inproceedings{liu2025kan,
  title={KAN: Kolmogorov--arnold networks},
  author={Liu, Ziming and Wang, Yixuan and Vaidya, Sachin and Ruehle, Fabian and Halverson, James and Soljacic, Marin and Hou, Thomas and Tegmark, Max},
  booktitle={International conference on learning representations},
  volume={2025},
  pages={70367--70413},
  year={2025}
}

@article{yu2024kan,
  title={Kan or mlp: A fairer comparison},
  author={Yu, Runpeng and Yu, Weihao and Wang, Xinchao},
  journal={arXiv preprint arXiv:2407.16674},
  year={2024}
}

@article{shukla2024comprehensive,
  title={A comprehensive and FAIR comparison between MLP and KAN representations for differential equations and operator networks},
  author={Shukla, Khemraj and Toscano, Juan Diego and Wang, Zhicheng and Zou, Zongren and Karniadakis, George Em},
  journal={Computer Methods in Applied Mechanics and Engineering},
  volume={431},
  pages={117290},
  year={2024},
  publisher={Elsevier}
}

@article{guo2025physics,
  title={Physics-informed Kolmogorov--Arnold network with Chebyshev polynomials for fluid mechanics},
  author={Guo, Chunyu and Sun, Lucheng and Li, Shilong and Yuan, Zelong and Wang, Chao},
  journal={Physics of Fluids},
  volume={37},
  number={9},
  year={2025},
  publisher={AIP Publishing}
}

@article{faroughi2026neural,
  title={Neural tangent kernel analysis to probe convergence in physics-informed neural solvers: PIKANs vs. PINNs},
  author={Faroughi, Salah A and Mostajeran, Farinaz},
  journal={Computers \& Mathematics with Applications},
  volume={215},
  pages={155--191},
  year={2026},
  publisher={Elsevier}
}

@article{sill1997monotonic,
  title={Monotonic networks},
  author={Sill, Joseph},
  journal={Advances in neural information processing systems},
  volume={10},
  year={1997}
}

\newpage
\appendix
\section{Theoretical Guarantee}
\label{sec-appendix:theorem}

We now show that the kernel constraints introduced in Section~\ref{sec:mckan} are preserved throughout the full depth of MC-KAN. Since the architecture enforces monotonicity and convexity independently along each input dimension, all results are stated and proved axis-wise.

\begin{lemma}[First-layer composition rule]
\label{lem:first-layer-composition}
Let $t:[a,b]\rightarrow[0,1]$ be monotone increasing and concave, and let $\phi:[0,1]\rightarrow\mathbb{R}_{>0}$ be monotone decreasing and convex. Then $\phi\circ t$ is positive, monotone decreasing, and convex on $[a,b]$.
\end{lemma}

\begin{proof}
Positivity follows immediately from $\phi>0$. For monotonicity, let $x\le y$. Since $t$ is monotone increasing, $t(x)\le t(y)$. Since $\phi$ is monotone decreasing, this implies $\phi(t(x))\ge \phi(t(y))$. Hence $\phi\circ t$ is monotone decreasing.
To prove convexity, let $x,y\in[a,b]$ and $\lambda\in[0,1]$. Since $t$ is concave, \[ t(\lambda x+(1-\lambda)y) \ge \lambda t(x)+(1-\lambda)t(y). \] Because $\phi$ is monotone decreasing, \[ \phi(t(\lambda x+(1-\lambda)y)) \le \phi\!\left(\lambda t(x)+(1-\lambda)t(y)\right). \] Using the convexity of $\phi$, \[ \phi\!\left(\lambda t(x)+(1-\lambda)t(y)\right) \le \lambda \phi(t(x))+(1-\lambda)\phi(t(y)). \] Therefore, \[ \phi(t(\lambda x+(1-\lambda)y)) \le \lambda \phi(t(x))+(1-\lambda)\phi(t(y)), \] which proves that $\phi\circ t$ is convex.
\end{proof}

\begin{lemma}[Deep-layer composition rule]
\label{lem:deep-layer-composition}
Let $g:[0,1]\rightarrow[0,1]$ be positive, monotone decreasing, convex, and twice differentiable on $(0,1)$. Let $\phi:[0,1]\rightarrow\mathbb{R}_{>0}$ be positive, monotone increasing, convex, and twice differentiable on $(0,1)$. Then $\phi\circ g$ is positive, monotone decreasing, and convex on the interior of its domain. These properties extend to the closed domain by continuity.
\end{lemma}

\begin{proof}
Positivity follows immediately from $\phi>0$. On the interior of the domain, the chain rule gives \[ (\phi\circ g)'=\phi'(g)g'. \] Since $\phi'\ge0$ and $g'\le0$, it follows that \[ (\phi\circ g)'\le0. \] Similarly, \[ (\phi\circ g)''=\phi''(g)(g')^2+\phi'(g)g''. \] Since $\phi''\ge0$, $(g')^2\ge0$, $\phi'\ge0$, and $g''\ge0$, both terms are nonnegative. Therefore, \[ (\phi\circ g)''\ge0. \] Thus $\phi\circ g$ is positive, monotone decreasing, and convex on the interior. Since $g$ and $\phi$ are continuous on $[0,1]$, the composition is continuous on the closed domain, and the monotonicity and convexity properties extend to the boundary.
\end{proof}

Layer 1 edge functions satisfy the decreasing coefficient conditions of Table~\ref{tab:constraints}, while all subsequent layers satisfy the corresponding increasing coefficient conditions. In addition, the reparameterization in ~\ref{sec-appendix:implementation} ensures that all Bernstein coefficients lie strictly in $(0,1)$. Therefore, every edge function maps $[0,1]$ into $(0,1)$, and all layer outputs, being averages of edge functions, also remain in $(0,1)$.

\begin{theorem}[MC-KAN preserves kernel constraints]
\label{thm:main}
Let MC-KAN have architecture \\
$[d,h_1,\ldots,h_{L-1},1]$, where $d$ denotes the number of kernel inputs, and define $h_0=d$ and $h_L=1$. Suppose that the normalized inputs are given by \[ t_i(\boldsymbol{\tau})=\left(\frac{\tau_i-\tau_i^{\min}}{\tau_i^{\max}-\tau_i^{\min}}\right)^{\alpha_i}, \qquad \alpha_i\in(0,1], \qquad i=1,\ldots,d . \] Suppose further that the Bernstein coefficients of layer 1 satisfy the decreasing conditions of Table~\ref{tab:constraints}, while the coefficients of layers $l\ge2$ satisfy the corresponding increasing conditions. Assume that all Bernstein coefficients are generated by the reparameterization in ~\ref{sec-appendix:implementation}, so that all coefficients lie strictly in $(0,1)$. We call a parameter vector $\boldsymbol{\theta}$ admissible if its Bernstein reparameterization satisfies the coefficient constraints and if its normalization parameters satisfy $\alpha_i\in(0,1]$ for all $i$. Then, for every input dimension $\tau_i$, the resulting kernel $K_{\boldsymbol{\theta}}(\boldsymbol{\tau})$ is positive, monotone decreasing, and convex with respect to $\tau_i$ on the closed input domain.
\end{theorem}

\begin{proof}
The proof proceeds by induction on the layer index.

\textbf{Base case.} Let \[ \mathbf{t}(\boldsymbol{\tau})=(t_1,\ldots,t_d) \] denote the normalized inputs. For each input dimension, \[ t_i(\boldsymbol{\tau})=\left(\frac{\tau_i-\tau_i^{\min}}{\tau_i^{\max}-\tau_i^{\min}}\right)^{\alpha_i}, \qquad \alpha_i\in(0,1]. \] The scalar map $x\mapsto x^{\alpha_i}$ is monotone increasing and concave on $[0,1]$ for every $\alpha_i\in(0,1]$. Therefore, each normalized input $t_i$ is monotone increasing and concave with respect to $\tau_i$ on $[\tau_i^{\min},\tau_i^{\max}]$.
Each layer-1 edge function is positive, monotone decreasing, and convex as a function of its scalar argument by the coefficient conditions of Table~\ref{tab:constraints}. Let \[ h_{p,q}^{(1)}=\phi_{p,q}^{(1)}(t_q). \] Fix an input dimension $\tau_i$. If $q\neq i$, then $t_q$ does not depend on $\tau_i$, so $h_{p,q}^{(1)}$ is constant with respect to $\tau_i$. Hence it is positive, monotone decreasing, and convex with respect to $\tau_i$.
If $q=i$, then $t_i$ is monotone increasing and concave with respect to $\tau_i$, while $\phi_{p,i}^{(1)}$ is positive, monotone decreasing, and convex. By Lemma~\ref{lem:first-layer-composition}, \[ h_{p,i}^{(1)}=\phi_{p,i}^{(1)}(t_i) \] is positive, monotone decreasing, and convex with respect to $\tau_i$ on the closed interval $[\tau_i^{\min},\tau_i^{\max}]$.
Because layer outputs are averages of edge functions, \[ z_p^{(1)}=\frac1{h_0}\sum_{q=1}^{h_0}h_{p,q}^{(1)}, \qquad h_0=d, \] the same properties hold for every layer-1 output. Moreover, since all Bernstein coefficients lie strictly in $(0,1)$, each edge function maps $[0,1]$ into $(0,1)$, and hence \[ z_p^{(1)}\in(0,1). \]
\textbf{Inductive step.} Assume that every output $z_q^{(l)}$ is positive, monotone decreasing, and convex with respect to $\tau_i$, and satisfies \[ z_q^{(l)}\in(0,1). \] The output of layer $l+1$ is \[ z_p^{(l+1)}=\frac1{h_l}\sum_{q=1}^{h_l}\phi_{p,q}^{(l+1)}\!\left(z_q^{(l)}\right), \] where each $\phi_{p,q}^{(l+1)}$ is positive, monotone increasing, and convex.
By Lemma~\ref{lem:deep-layer-composition}, every term \[ \phi_{p,q}^{(l+1)}\!\left(z_q^{(l)}\right) \] is positive, monotone decreasing, and convex with respect to $\tau_i$. Since sums and positive averages preserve positivity, monotonicity, and convexity, $z_p^{(l+1)}$ inherits these properties.
Furthermore, because $z_q^{(l)}\in(0,1)$ and all Bernstein coefficients of $\phi_{p,q}^{(l+1)}$ lie strictly in $(0,1)$, each edge output also lies in $(0,1)$. Therefore their average satisfies \[ z_p^{(l+1)}\in(0,1). \]
By induction, every layer output is positive, monotone decreasing, and convex with respect to $\tau_i$, and remains in $(0,1)$. Since the kernel output is the final network output, \[ K_{\boldsymbol{\theta}}=z^{(L)}, \] the result follows.
\end{proof}

\begin{remark}
Theorem~\ref{thm:main} holds for arbitrary depth, width, and input dimension. The resulting kernel is positive, monotone decreasing, and convex along each input dimension by construction. For multidimensional kernels, these properties are enforced independently along each axis and do not imply joint convexity over the full input domain.
\end{remark}

\section{Implementation}
\label{sec-appendix:implementation}

To enforce the coefficient conditions of
Table~\ref{tab:constraints}, the Bernstein coefficients are generated from unconstrained trainable parameters through a differentiable reparameterization. For a Bernstein polynomial of degree $n$, let $\rho_k \in \mathbb{R}$, $k=0,\ldots,n-2$, denote unconstrained parameters. We first define non-negative increments,
\begin{equation}
    s_k
    =
    \operatorname{ReLU}\bigl(
        \operatorname{softplus}(\rho_k)-\varepsilon
    \bigr),
\end{equation}
where $\varepsilon>0$ is a small constant. For the first layer, a decreasing gap sequence is constructed as,
\begin{equation}
    g_{n-1}
    =
    -\operatorname{softplus}(\rho_{\mathrm{anchor}}),
\end{equation}
\begin{equation}
    g_k
    =
    g_{k+1}-s_k,
    \qquad
    k=n-2,\ldots,0.
\end{equation}
The coefficient range is controlled through,
\begin{equation}
    c_{\mathrm{level}}
    =
    \operatorname{sigmoid}(\rho_{\mathrm{level}}),
\end{equation}
\begin{equation}
    c_{\min}
    =
    c_{\mathrm{level}}
    \operatorname{sigmoid}(\rho_{\mathrm{floor}}),
\end{equation}
which satisfy
$
0<c_{\min}<c_{\mathrm{level}}<1.
$ The gap sequence is scaled according to,
\begin{equation}
    d
    =
    c_{\mathrm{level}}-c_{\min},
    \qquad
    S
    =
    \sum_{k=0}^{n-1}|g_k|,
\end{equation}
\begin{equation}
    \tilde g_k
    =
    g_k \frac{d}{S}.
\end{equation}
The Bernstein coefficients are then constructed recursively as,
\begin{equation}
    c_0
    =
    c_{\mathrm{level}},
\end{equation}
\begin{equation}
    c_{k+1}
    =
    c_k+\tilde g_k,
    \qquad
    k=0,\ldots,n-1.
\end{equation}
This construction yields,
\[
c_0 \ge c_1 \ge \cdots \ge c_n = c_{\min} > 0,
\]
with non-negative second differences. Consequently, the resulting
Bernstein polynomial is positive, monotone decreasing, and convex.

For layers $l\ge2$, monotone increasing edge functions are obtained by
reversing the coefficient sequence,
\begin{equation}
    c_k^{\uparrow}
    =
    c_{n-k},
    \qquad
    k=0,\ldots,n,
\end{equation}
which preserves convexity while reversing monotonicity. The learnable normalization exponent $\alpha_i$ is parameterized through an unconstrained variable $\tilde{\alpha}_i$ using the sigmoid transformation,
\begin{equation}
\alpha_i=\sigma(\tilde{\alpha}_i)=\frac{1}{1+e^{-\tilde{\alpha}_i}},
\end{equation}
which guarantees that $\alpha_i\in(0,1)$ throughout training. The network's trainable parameters are,
\begin{equation}
\boldsymbol{\theta}
= \Bigl\{
\rho_k^{(l,p,q)},
\,
\rho_{\mathrm{anchor}}^{(l,p,q)},
\,
\rho_{\mathrm{level}}^{(l,p,q)},
\,
\rho_{\mathrm{floor}}^{(l,p,q)},
\,
\tilde{\alpha_i}
\Bigr\},
\end{equation}
To improve numerical stability, Bernstein basis functions are evaluated
in log-space,
\begin{equation}
    \log B_{k,n}(t)
    =
    \log\binom{n}{k}
    +
    k\log t
    +
    (n-k)\log(1-t),
\end{equation}
with inputs clamped to $t \in [10^{-8},\,1-10^{-8}]$.

\section{Numerical Solver Details}
\label{sec:appendix-numerical}

\subsection{Experiment I: 1D Volterra IDE}\label{sec:appendix-p1}

\begin{figure}[h!]
    \centering
    \includegraphics[width=0.5\linewidth]{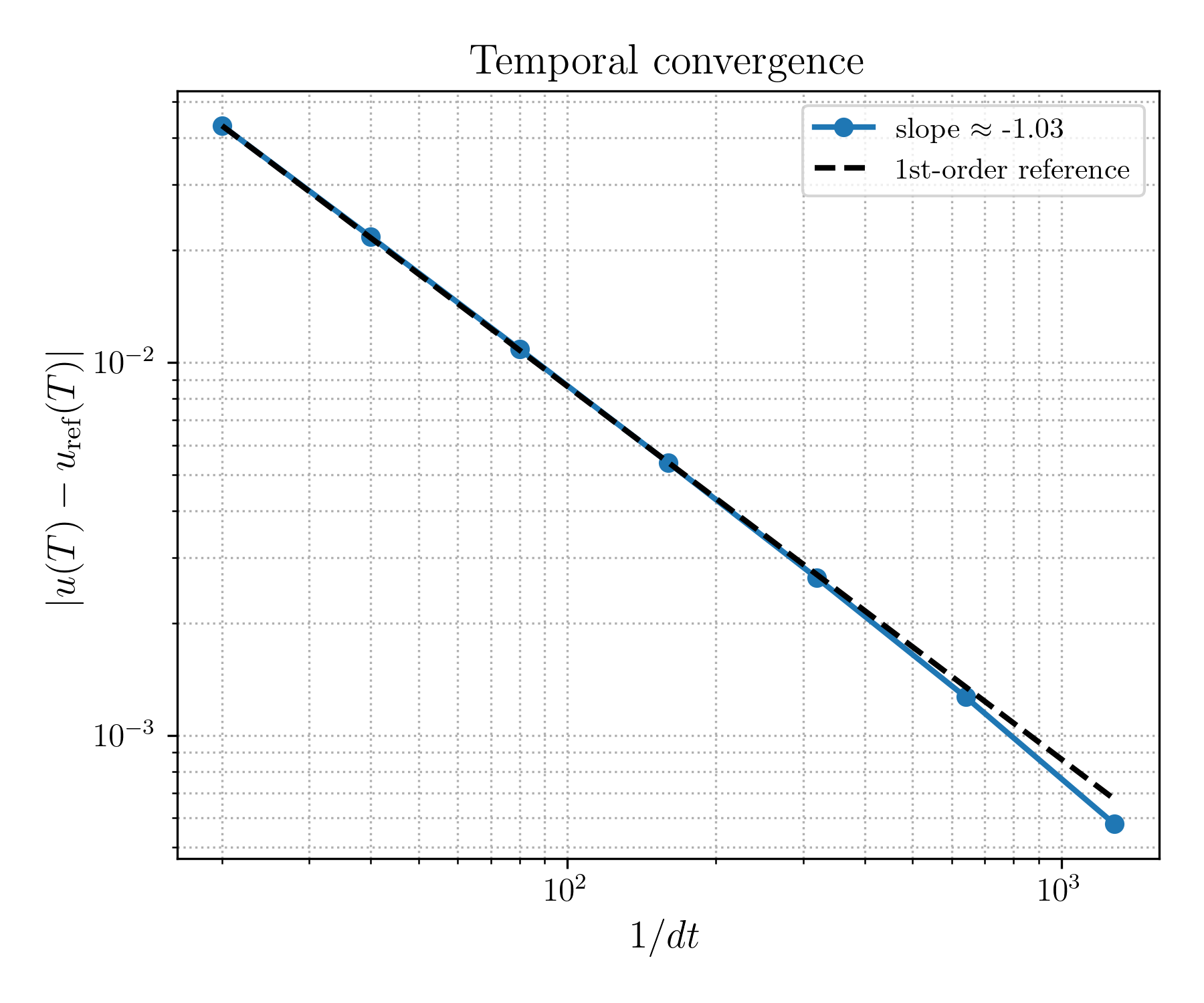}
    \caption{Temporal convergence of the Volterra solver. Endpoint error $|y(T)-y_{\mathrm{ref}}(T)|$ against a fine-grid reference, plotted versus $1/\Delta t$ on log--log axes. The fitted slope of approximately $1$ confirms the first-order forward Euler time-stepping scheme.} 
    \label{fig:appendix-p1-convergence}
\end{figure}

\begin{table}[ht]
\centering
\caption{Configuration for Experiment I: 1D Volterra IDE.}
\label{tab:exp1-config}
\begin{tabular}{lll}
\hline
Category & Parameter & Value \\
\hline

\multicolumn{3}{c}{\textbf{Problem setup}} \\
\hline
Kernel strength & $\kappa$ & $1.5$ \\
Final time & $T$ & $5.0$ s \\
Initial condition & $u(0)$ & $1.0$ \\
Reference kernel & $K(t)$ & $e^{-t}$ \\
\hline

\multicolumn{3}{c}{\textbf{Ground-truth generation}} \\
\hline
Time integrator & -- & Forward Euler \\
Memory quadrature & -- & Trapezoidal rule \\
Time step & $\Delta t$ & $5\times10^{-3}$ s \\
Total snapshots & $N_t$ & $1001$ \\
\hline

\multicolumn{3}{c}{\textbf{Training dataset}} \\
\hline
Observed snapshots & $N_{\mathrm{obs}}$ & $501$ \\
Sampling strategy & -- & Uniform subsampling \\
Noise model & -- & Additive Gaussian \\
Noise levels & $\sigma$ & $0.00,\;0.05,\;0.10,\;0.15$ \\
Observation model & -- & $u^{\mathrm{obs}} = u + \sigma\varepsilon$ \\
Noise distribution & -- & $\varepsilon \sim \mathcal{N}(0,1)$ \\
\hline

\multicolumn{3}{c}{\textbf{Differentiable solver}} \\
\hline
Time integrator & -- & Forward Euler \\
Memory quadrature & -- & Trapezoidal rule \\
Convolution evaluation & -- & Direct summation \\
Time step & $\Delta t$ & $5\times10^{-3}$ s \\
\hline

\multicolumn{3}{c}{\textbf{Optimization}} \\
\hline
Loss function & -- & MSE \\
Adam epochs & -- & $2000$ \\
Adam learning rate & -- & $3\times10^{-3}$ \\
L-BFGS epochs & -- & $200$ \\
L-BFGS learning rate & -- & $0.5$ \\
Line search & -- & Strong Wolfe \\
Cheb-KAN constraint weight & $\lambda$ & $10^{-1}$ \\
\hline
\multicolumn{3}{c}{\textbf{Symbolic regression (PySR)}} \\
\hline
Maximum expression size & \texttt{maxsize} & $20$ \\
Iterations & \texttt{niterations} & $40$ \\
Binary operators & -- & $+,\;\times$ \\
Unary operators & -- & $\cos,\;\exp,\;\sin,\;\tanh$ \\
Elementwise loss & -- & Squared error \\
Model selection & -- & Score \\
\hline
\end{tabular}
\end{table}

\begin{figure}[h!]
    \centering
    \includegraphics[width=1\linewidth]{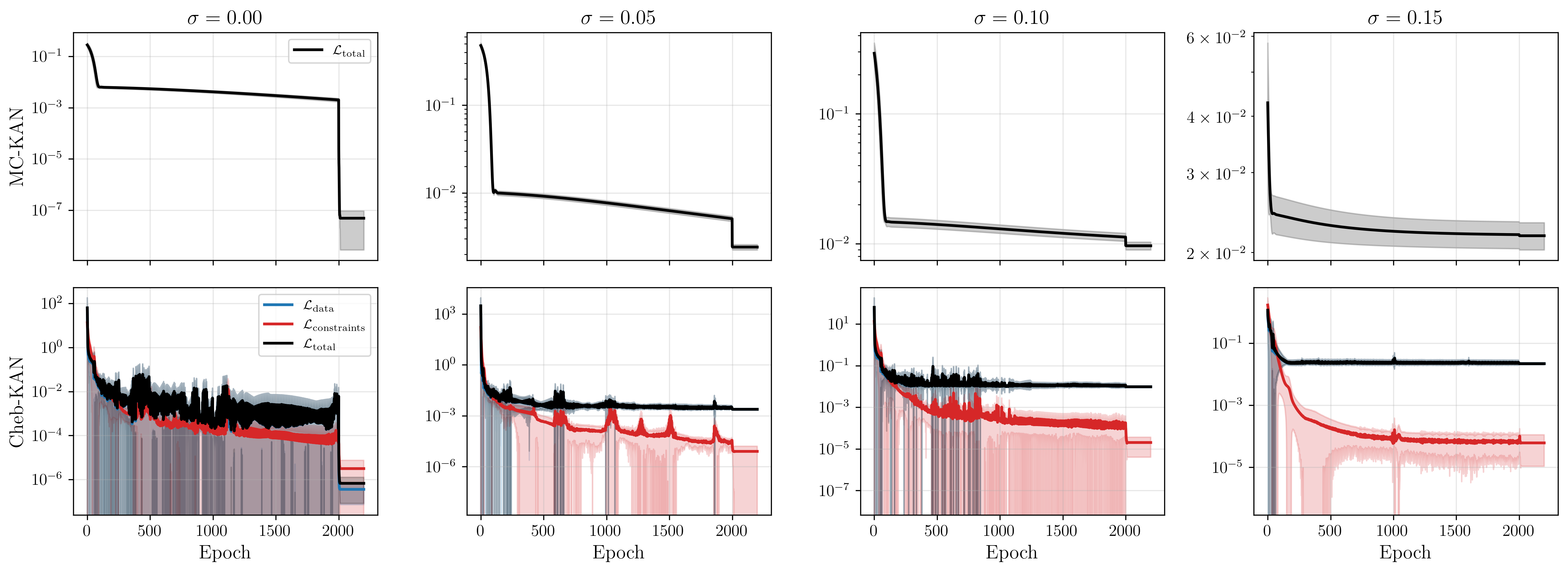}
    \caption{Training loss histories in Experiment I for the same best-performing architectures shown in Fig.~\ref{fig:p1-fig2}. Rows correspond to MC-KAN (top) and Cheb-KAN (bottom), and columns correspond to noise levels $\sigma=0.00$, $0.05$, $0.10$, and $0.15$ (left to right). The iteration axis combines the Adam and L-BFGS optimization stages into a single continuous trajectory. Curves show the mean over five seeds, with shaded bands indicating $\pm$ one standard deviation.}
    \label{fig:appendix-p1-fig1}
\end{figure}

\begin{figure}[h!]
    \centering
    \includegraphics[width=1\linewidth]{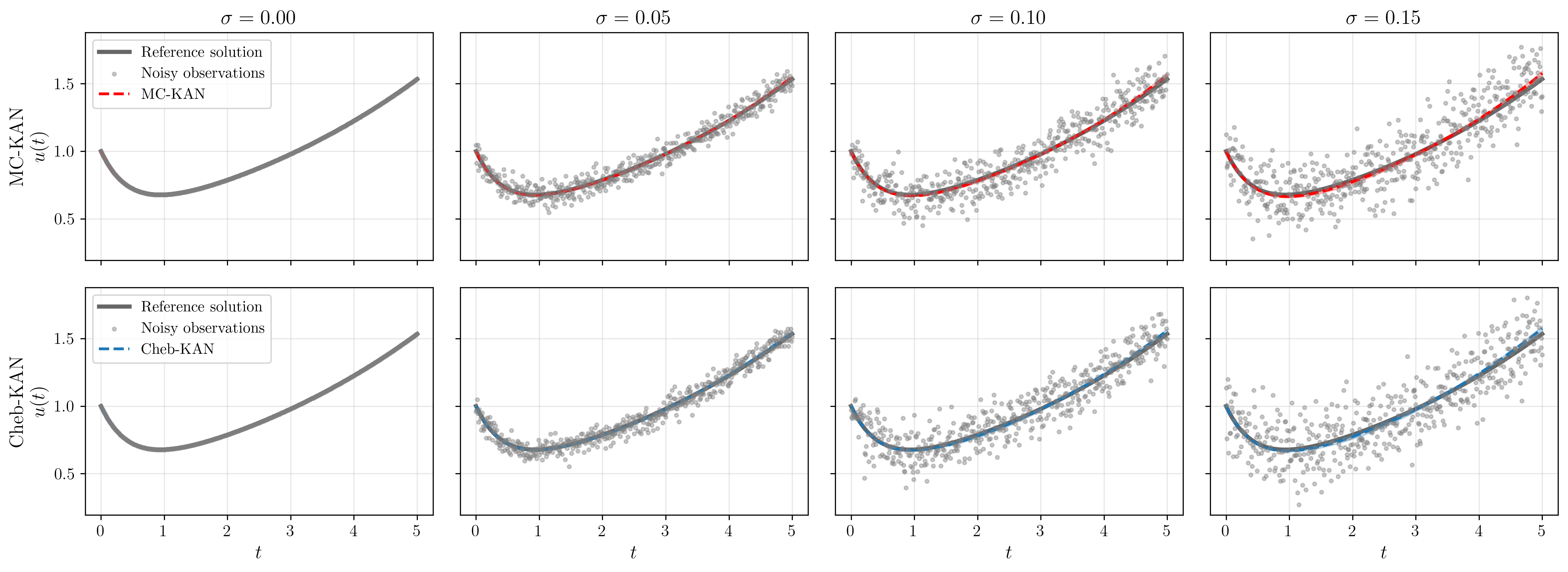}
    \caption{Predicted solution trajectories in Experiment I under increasing noise, using the same best-performing architecture (lowest $\mathcal{E}(K)$) at each noise level as in Fig.~\ref{fig:p1-fig2}. Rows correspond to MC-KAN (top) and Cheb-KAN (bottom), and columns correspond to noise levels $\sigma=0.00$, $0.05$, $0.10$, and $0.15$ (left to right). Each panel shows the observed solution (black), the pointwise mean prediction over five seeds (dashed; MC-KAN in red, Cheb-KAN in blue), and a shaded $\pm$ one standard deviation band.}
    \label{fig:appendix-p1-fig2}
\end{figure}

\newpage
\subsection{Experiment II: 1D Viscoelastic Wave PIDE}\label{sec:appendix-p2}

\begin{figure}[h!]
    \centering
    \includegraphics[width=1\linewidth]{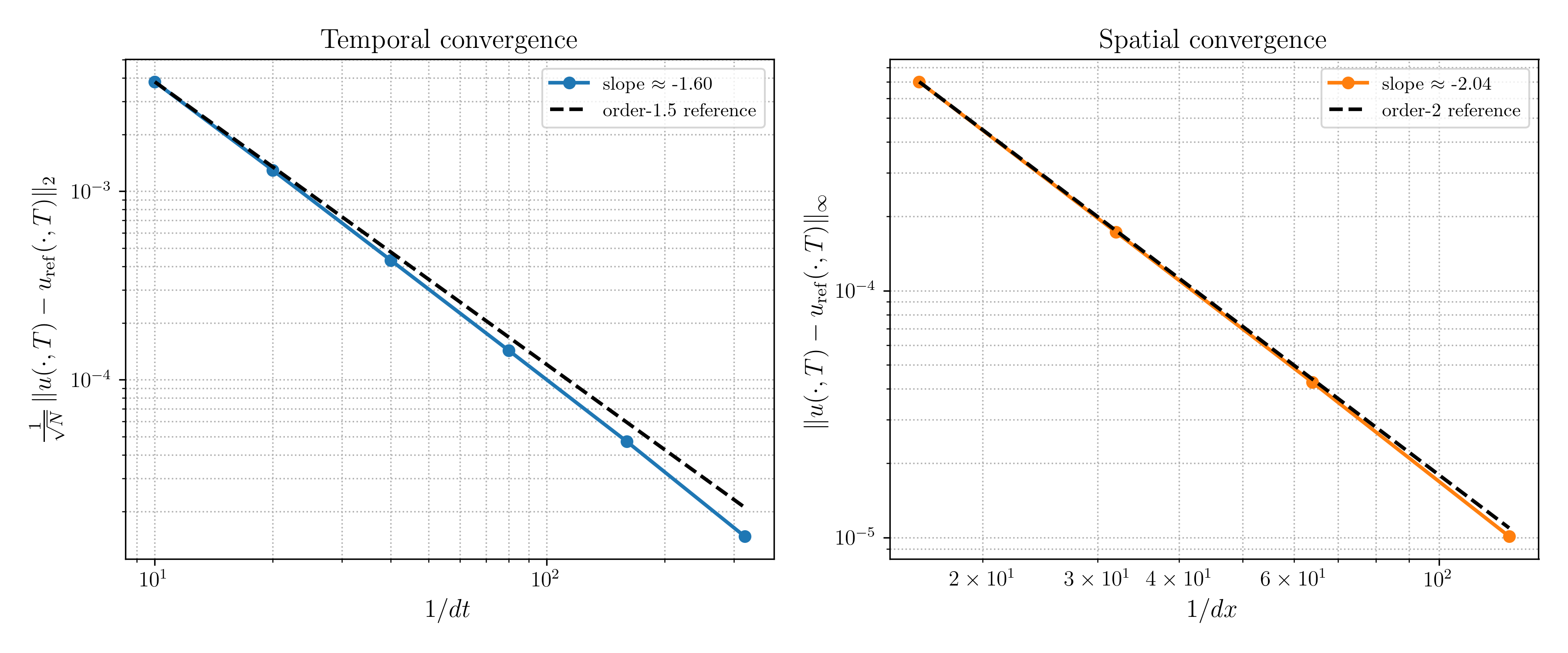}
    \caption{Convergence of the PIDE solver against a fine reference. Left: temporal convergence (spatial grid fixed), RMS error versus $1/\Delta t$ on log--log axes. Slope $1.6$ reflects the reduced temporal order induced by the weakly singular ($\beta=0.5$) memory kernel. Right: spatial convergence ($\Delta t$ fixed), max-norm error versus $1/\Delta x$. The fitted slope of approximately $2$ confirms the expected second-order accuracy of the central finite-difference discretization.}
    \label{fig:appendix-p2-convergence}
\end{figure}

\begin{table}[ht]
\centering
\caption{Configuration for Experiment II: 1D viscoelastic wave PIDE.}
\label{tab:exp2-config}
\begin{tabular}{lll}
\hline
Category & Parameter & Value \\
\hline

\multicolumn{3}{c}{\textbf{Problem setup}} \\
\hline
Rod length & $L$ & $1.0$ m \\
Mass density & $\rho$ & $1.0$ kg/m$^3$ \\
Final time & $T$ & $5.0$ s \\
Reference kernel & $K(t)$ & $\exp\!\left[-(t/\tau_0)^\beta\right]$ \\
Relaxation time & $\tau_0$ & $1.5$ s \\
Stretching exponent & $\beta$ & $0.5$ \\
\hline

\multicolumn{3}{c}{\textbf{Ground-truth generation}} \\
\hline
Spatial discretization & -- & Central finite differences \\
Time integrator & -- & Newmark--$\beta$ \\
Memory quadrature & -- & Trapezoidal rule \\
Spatial nodes & $N_x$ & $200$ \\
Total snapshots & $N_t$ & $1001$ \\
Spatial step & $\Delta x$ & $5.03\times10^{-3}$ m \\
Time step & $\Delta t$ & $5\times10^{-3}$ s \\
Newmark parameter & $\gamma$ & $1/2$ \\
Newmark parameter & $\beta_N$ & $1/4$ \\
\hline

\multicolumn{3}{c}{\textbf{Training dataset}} \\
\hline
Observed spatial locations & $N_x^{\mathrm{obs}}$ & $50$ \\
Observed temporal snapshots & $N_s$ & $\{101,\,51,\,21,\,11\}$ \\
Sampling strategy & -- & Uniform in space and time \\
Noise model & -- & None \\
\hline

\multicolumn{3}{c}{\textbf{Differentiable solver}} \\
\hline
Spatial discretization & -- & Central finite differences \\
Time integrator & -- & Newmark--$\beta$ \\
Memory quadrature & -- & Trapezoidal rule \\
Spatial nodes & $N_x$ & $200$ \\
Total snapshots & $N_t$ & $1001$ \\
Time step & $\Delta t$ & $5\times10^{-3}$ s \\
\hline

\multicolumn{3}{c}{\textbf{Optimization}} \\
\hline
Loss function & -- & MSE \\
Adam epochs & -- & $2000$ \\
Adam learning rate & -- & $3\times10^{-3}$ \\
L-BFGS epochs & -- & $200$ \\
L-BFGS learning rate & -- & $0.5$ \\
Line search & -- & Strong Wolfe \\
Cheb-KAN constraint weight & $\lambda$ & $3\times10^{-1}$ \\
\hline
\multicolumn{3}{c}{\textbf{Symbolic regression (PySR)}} \\
\hline
Maximum expression size & \texttt{maxsize} & $20$ \\
Iterations & \texttt{niterations} & $200$ \\
Binary operators & -- & $+,\;\times,\;\wedge$ \\
Unary operators & -- & $\exp,\;\log,\;\cos,\;\tanh$ \\
Elementwise loss & -- & Squared error \\
Model selection & -- & Score \\
\hline
\end{tabular}
\end{table}

\begin{figure}[ht]
    \centering
    \includegraphics[width=1\linewidth]{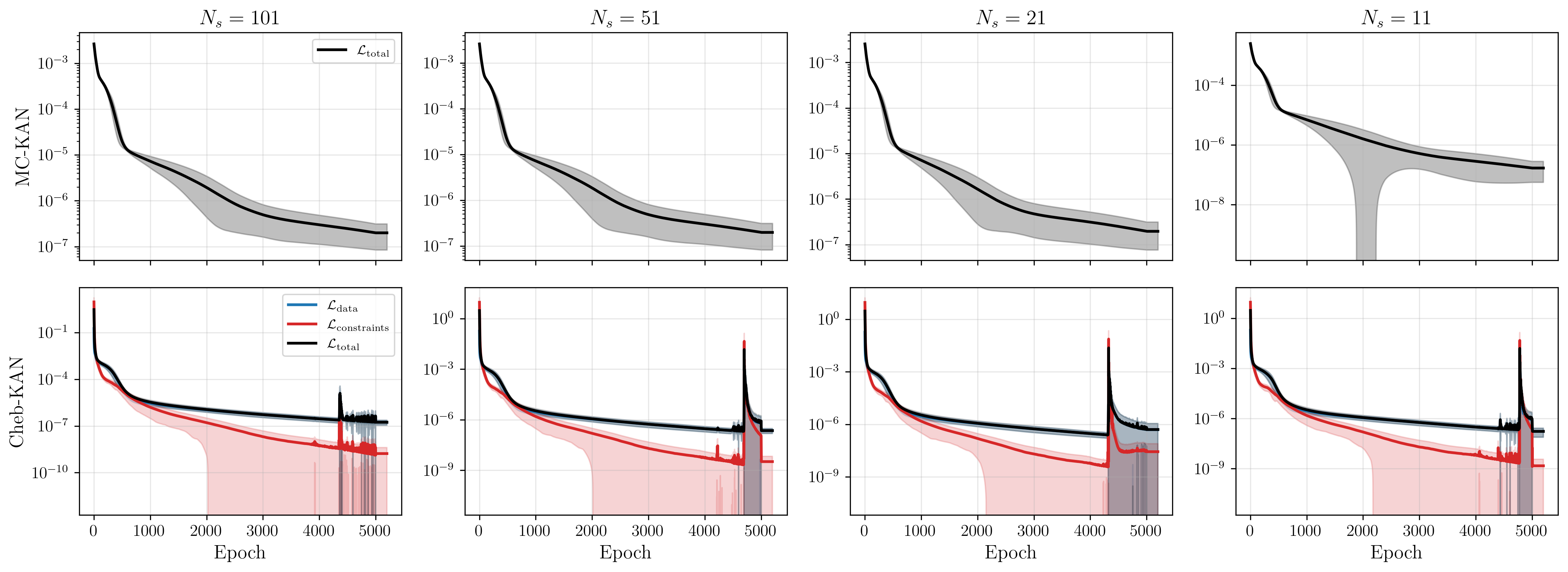}
    \caption{Training loss histories in Experiment II for the same models shown in Fig.~\ref{fig:p2-fig1}. Rows correspond to MC-KAN (top) and Cheb-KAN (bottom), and columns correspond to temporal sampling levels $N_s=101$, $51$, $21$, and $11$ (left to right). The iteration axis combines the Adam and L-BFGS optimization stages into a single continuous trajectory. Curves show the mean over five random seeds, with shaded bands indicating $\pm$ one standard deviation.}
    \label{fig:p2-appendix-fig1}
\end{figure}

\begin{figure}[ht]
    \centering
    \includegraphics[width=1\linewidth]{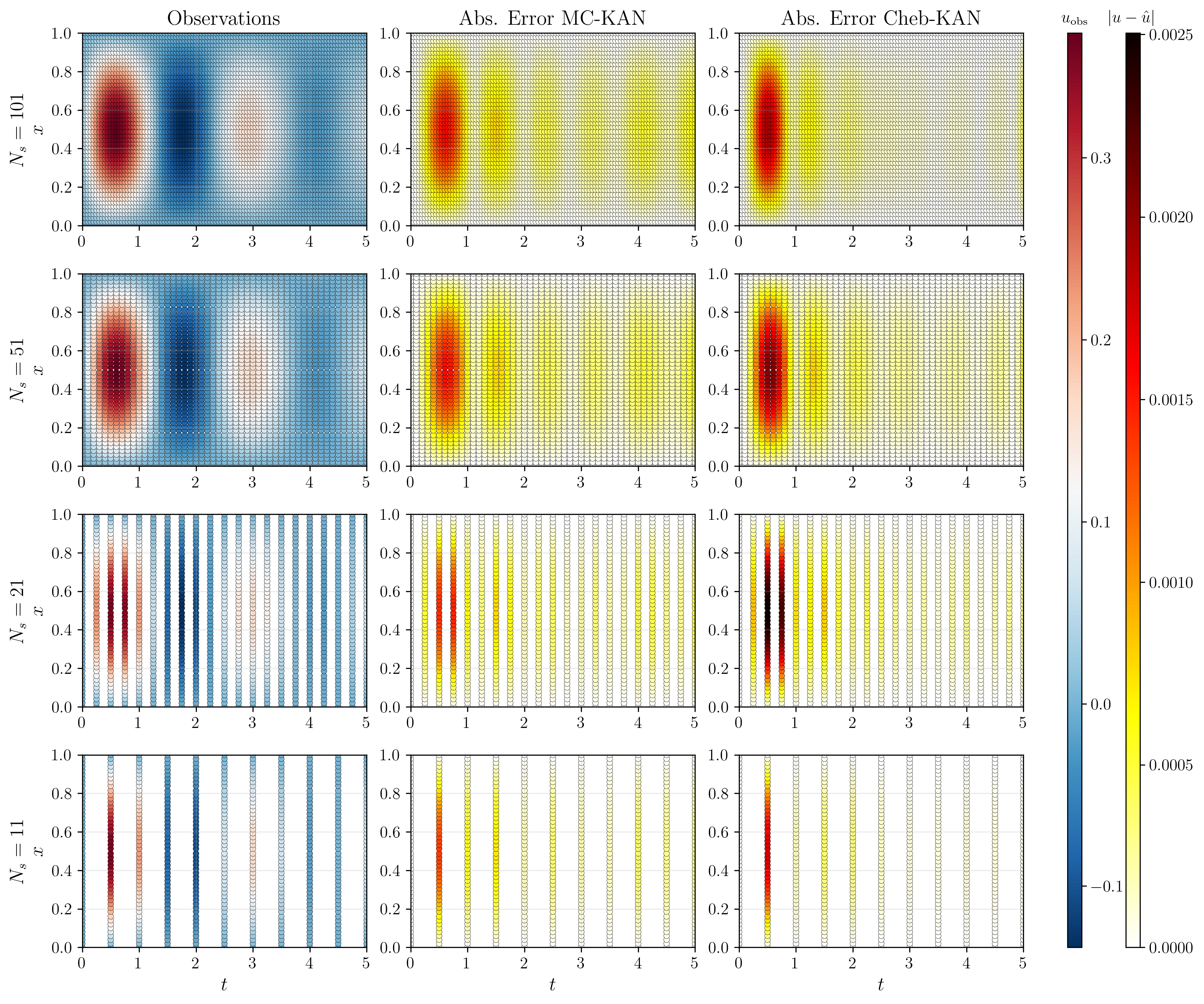}
    \caption{Sparse observations and solution reconstruction errors in Experiment II for the same models shown in Fig.~\ref{fig:p2-fig1}. Rows correspond to temporal sampling levels $N_s=101$, $51$, $21$, and $11$ (top to bottom). The first column shows the observed solution values at the sampled space--time locations. The second and third columns show the pointwise mean absolute solution error, $\frac{1}{5}\sum_{i=1}^{5}\lvert u_{\mathrm{pred}}^{(i)}-u_{\mathrm{true}}\rvert$, for MC-KAN and Cheb-KAN, respectively, where the average is taken over the five random seeds. All panels use the same color scale within each quantity for clear comparison.}
    \label{fig:p2-appendix-fig2}
\end{figure}

\newpage
\subsection{Experiment III: 2D Nonlocal Reaction-Diffusion Equation}\label{appendix-p3}

\subsubsection{Initial Condition Construction}\label{appendix:p3-ic}

The initial condition is generated as a smooth random field in order to excite a broad range of spatial modes while avoiding grid-scale noise. Let $\zeta(x,y)$ denote a realization of spatial white noise with independent standard normal samples. Its Fourier transform
$\widehat{\zeta}(\mathbf{k})$ is multiplied by a Gaussian low-pass filter,
\begin{equation}
    \widehat{u}_0(\mathbf{k})
    =
    \widehat{\zeta}(\mathbf{k})
    \exp\!\left(
        -\frac{\|\mathbf{k}\|^2}
               {2\sigma_s^2}
    \right),
\end{equation}
where $\mathbf{k}=(k_x,k_y)$ is the wavevector and
$\sigma_s$ controls the smoothness of the resulting field. The initial condition is then obtained by inverse Fourier transform,
\begin{equation}
    u_0(x,y)
    =
    \mathcal{F}^{-1}
    \!\left[
        \widehat{u}_0(\mathbf{k})
    \right].
\end{equation}
Finally, the field is normalized to have prescribed standard deviation $a$ and zero mean,
\begin{equation}
    u_0(x,y)
    \leftarrow
    a\,
    \frac{
        u_0(x,y)
    }{
        \operatorname{std}(u_0)
    },
\end{equation}
followed by
\begin{equation}
    u_0(x,y)
    \leftarrow
    u_0(x,y)
    -
    \frac{1}{|\Omega|}
    \int_{\Omega}
    u_0(x,y)\,dx\,dy.
\end{equation}
In this work, we use
\[
\sigma_s = 0.1,
\qquad
a = 0.5.
\]
The same initial condition is used for all simulations.
This initialization serves two purposes. From a physical perspective, it represents a spatially heterogeneous state containing fluctuations across a range of length scales, which is characteristic of many phase
separation and pattern-formation processes. From a numerical perspective, the broadband excitation ensures that the evolving field contains sufficient information about the underlying nonlocal interactions, producing gradients of adequate magnitude during backpropagation to enable reliable kernel identification.

\begin{figure}[h!]
    \centering
    \includegraphics[width=1\linewidth]{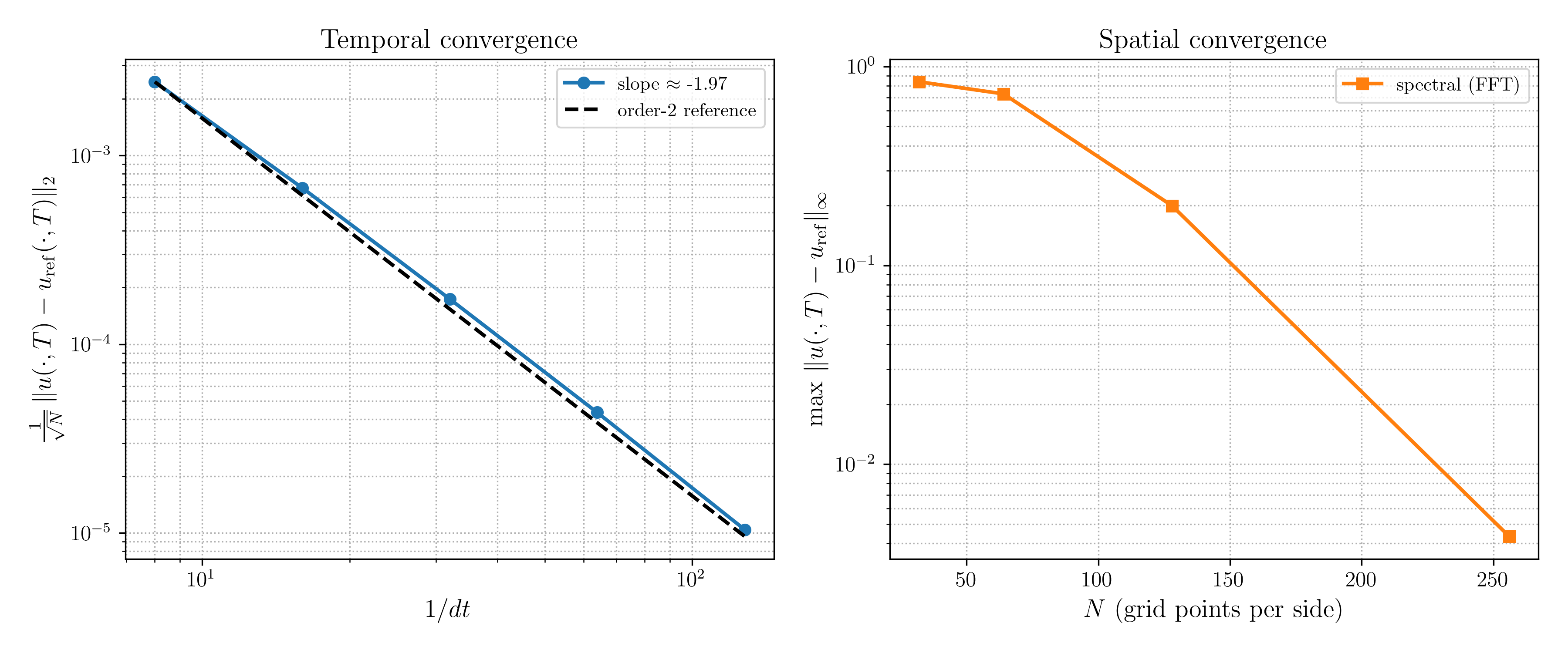}
    \caption{Convergence of the 2D spectral solver against a fine-grid reference. Left: temporal convergence (spatial grid fixed), RMS error versus $1/\Delta t$ on log--log axes, slope $\approx 2$ confirming the second-order Crank--Nicolson/Adams--Bashforth (CN/AB2) scheme. Right: spatial convergence ($\Delta t$ fixed), max-norm error versus the number of grid points $N$ on semilog axes. The accelerating decay is consistent with the spectral convergence expected from the Fourier discretization.}
    \label{fig:appendix-p3-convergence}
\end{figure}

\begin{table}[ht]
\centering
\caption{Configuration for Experiment III: 2D nonlocal reaction-diffusion equation.}
\label{tab:exp3-config}
\begin{tabular}{lll}
\hline
Category & Parameter & Value \\
\hline

\multicolumn{3}{c}{\textbf{Problem setup}} \\
\hline
Domain length ($x$) & $L_x$ & $10$ \\
Domain length ($y$) & $L_y$ & $10$ \\
Final time & $T$ & $6.0$ \\
Diffusion coefficient & $D$ & $0.001$ \\
Interaction strength & $\alpha$ & $0.12$ \\
Boundary conditions & -- & Periodic \\
\hline

\multicolumn{3}{c}{\textbf{Reference kernel}} \\
\hline
Interaction length & $\tau_x$ & $3.0$ \\
Interaction length & $\tau_y$ & $1.0$ \\
Decay exponent & $\beta_x$ & $0.4$ \\
Decay exponent & $\beta_y$ & $0.8$ \\
Mixed exponent & $\gamma_x$ & $0.5$ \\
Mixed exponent & $\gamma_y$ & $0.5$ \\
Coupling parameter & $\delta$ & $0.5$ \\
\hline

\multicolumn{3}{c}{\textbf{Initial condition}} \\
\hline
Construction & -- & Smoothed random field \\
Smoothing parameter & $\sigma_s$ & $0.1$ \\
Amplitude & $a$ & $0.5$ \\
\hline

\multicolumn{3}{c}{\textbf{Ground-truth generation}} \\
\hline
Spatial discretization & -- & Fourier spectral \\
Time integrator & -- & Semi-implicit IMEX \\
Convolution evaluation & -- & FFT-based \\
Spatial grid & $N_x\times N_y$ & $256\times256$ \\
Time step & $\Delta t$ & $0.05$ \\
\hline

\multicolumn{3}{c}{\textbf{Training dataset}} \\
\hline
Observed spatial grid & $N_x^{\mathrm{obs}}\times N_y^{\mathrm{obs}}$ & $32\times32$ \\
Observation times (coarse) & $t_{\mathrm{obs}}^{(1)}$ & $\{0,1,2,3,4,5,6\}$ \\
Observation times (dense) & $t_{\mathrm{obs}}^{(2)}$ & $\{0,0.25,0.50,\ldots,6.00\}$ \\
Sampling strategy & -- & Uniform spatial subsampling \\
Noise model & -- & Additive Gaussian \\
Noise levels & $\sigma$ & $\{0,0.02,0.04,0.06,0.08,0.10,0.15\}$ \\
Observation model & -- & $u^{\mathrm{obs}} = u + \sigma\varepsilon$ \\
Noise distribution & -- & $\varepsilon \sim \mathcal{N}(0,1)$ \\
\hline

\multicolumn{3}{c}{\textbf{Differentiable solver}} \\
\hline
Spatial discretization & -- & Fourier spectral \\
Time integrator & -- & Semi-implicit IMEX \\
Convolution evaluation & -- & FFT-based \\
Spatial grid & $N_x\times N_y$ & $256\times256$ \\
Time step & $\Delta t$ & $0.05$ \\
Loss evaluation & -- & Observed points only \\
\hline

\multicolumn{3}{c}{\textbf{Optimization}} \\
\hline
Optimizer & -- & Adam \\
Learning rate & -- & $10^{-1}$ for MC-KAN; $10^{-3}$ for Cheb-KAN \\
Maximum epochs (noise-free) & -- & $120\,000$\\
Maximum epochs (noisy) & -- & $30\,000$\\
Learning-rate scheduler & -- & StepLR \\
Scheduler step size & -- & $20\,000$ \\
Scheduler decay & -- & $0.7$ \\

Cheb-KAN constraint weight & $\lambda$ & $10^{-2}$ \\
\hline

\multicolumn{3}{c}{\textbf{Symbolic regression}} \\
\hline
Maximum expression size & \texttt{maxsize} & $20$ \\
Iterations & \texttt{niterations} & $200$ \\
Binary operators & -- & $+,\;\times,\;\wedge$ \\
Unary operators & -- & $\exp,\;\log$ \\
Elementwise loss & -- & Squared error \\
Model selection & -- & Score \\
\hline
\end{tabular}
\end{table}

\begin{figure}[ht]
    \centering
    \includegraphics[width=1\linewidth]{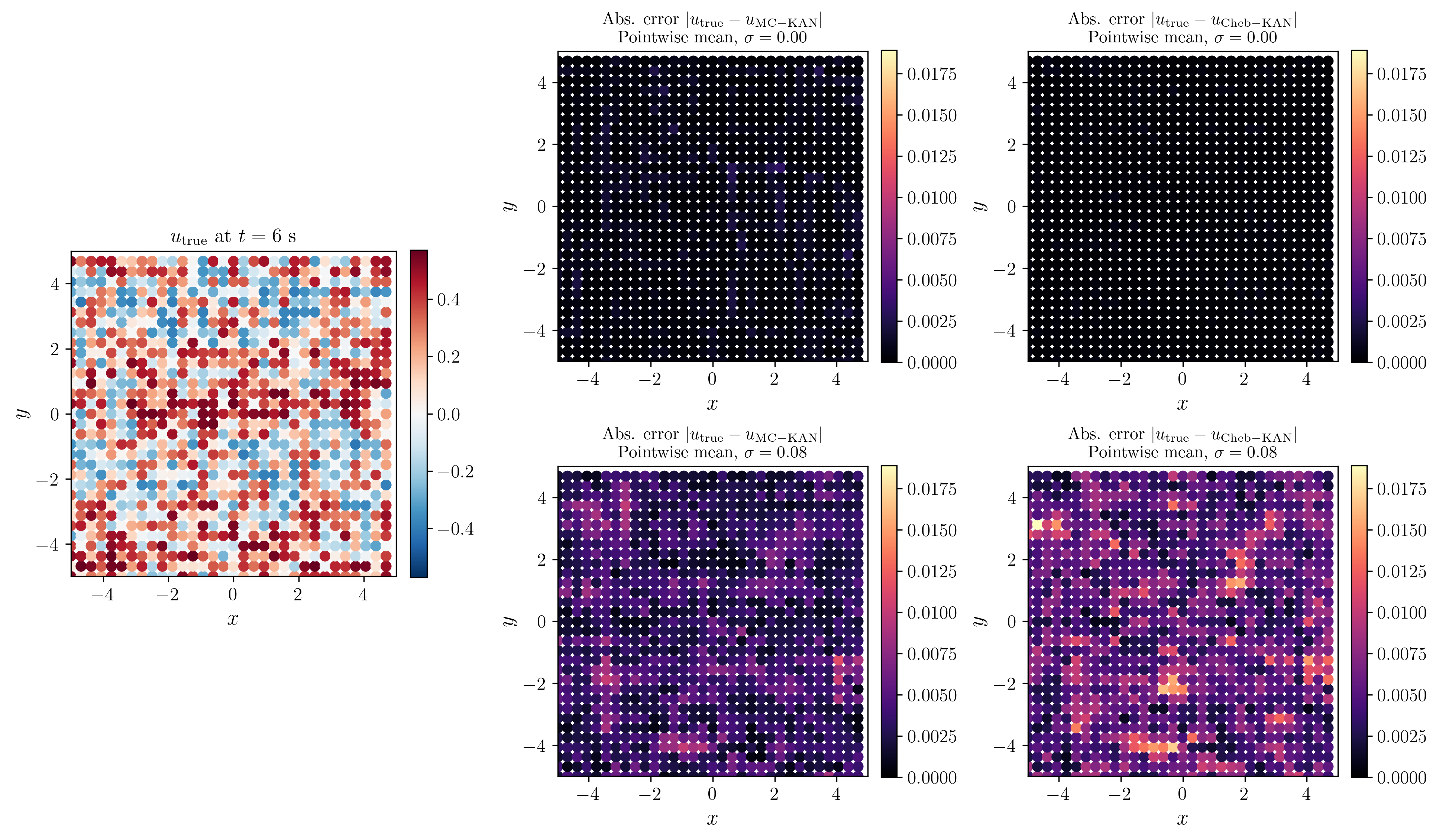}
    \caption{Final-time solution reconstruction error in Experiment III on the $32\times32$ observation grid using 7 temporal snapshots. The left column shows the ground-truth solution $u(\mathbf{x},T)$ at the final time $T=6$. The center and right columns show the pointwise mean absolute error, $\frac{1}{5}\sum_{i=1}^{5}\lvert u_{\mathrm{pred}}^{(i)}-u_{\mathrm{true}}\rvert$, for MC-KAN and Cheb-KAN, respectively, where the average is taken over the five noise realizations with fixed initialization. The top row corresponds to noise-free observations ($\sigma=0.00$), and the bottom row to $\sigma=0.08$. Both models use linear input normalization.}
    \label{fig:p3-appendix-fig1}
\end{figure}

\begin{figure}[ht]
    \centering
    \includegraphics[width=1\linewidth]{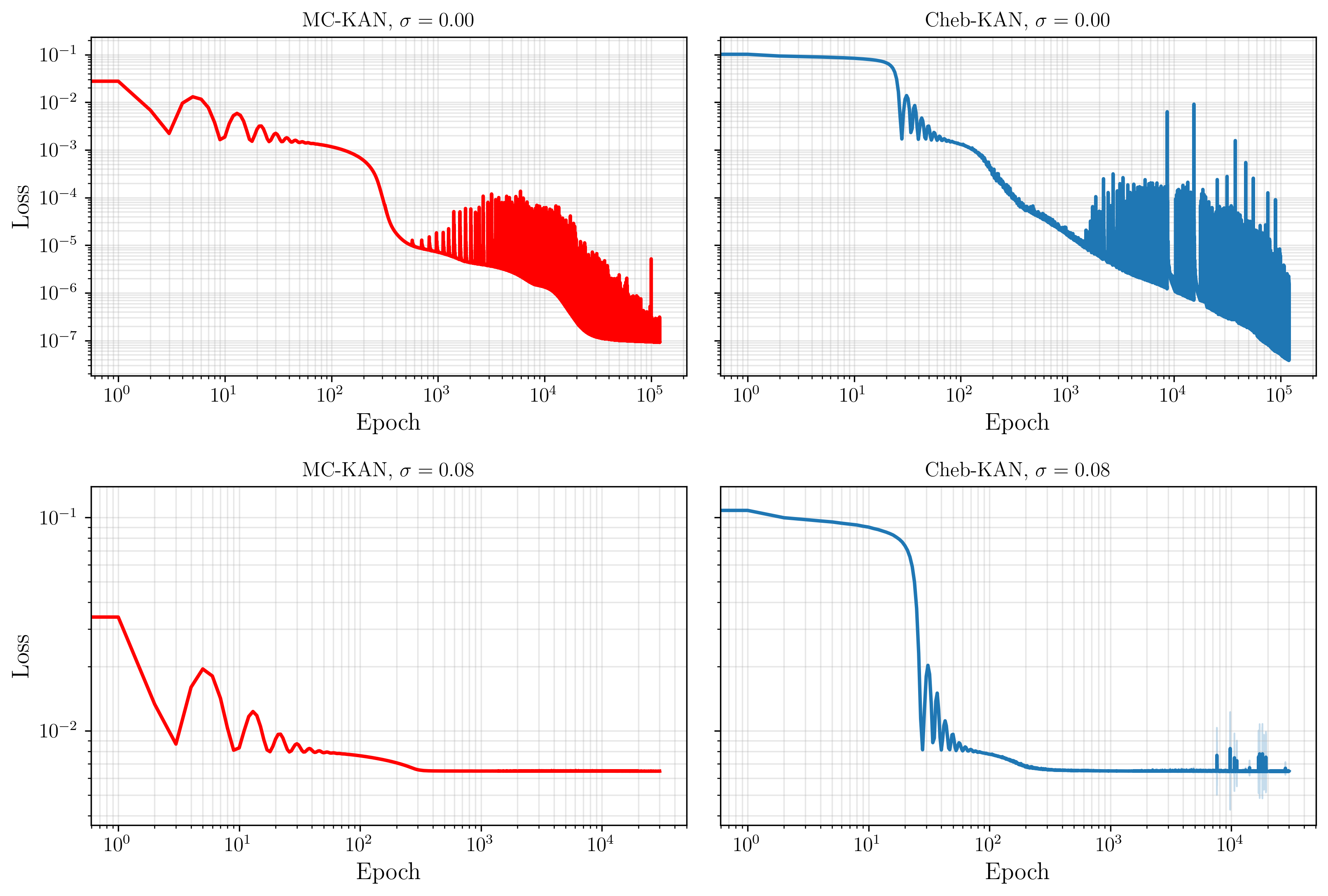}
    \caption{Training loss histories in Experiment III using 25 temporal snapshots. Curves show the mean over five noise realizations with fixed initialization, with shaded bands indicating $\pm$ one standard deviation. Rows correspond to observation noise levels $\sigma=0.00$ (top) and $\sigma=0.08$ (bottom), while columns correspond to MC-KAN (power-fixed input normalization with exponent $\alpha_1=\alpha_2=0.75$; left) and Cheb-KAN (linear input normalization; right).}
    \label{fig:p3-appendix-fig2}
\end{figure}

\end{document}